\documentclass[11pt]{article}

\usepackage[utf8]{inputenc}
\usepackage{amsmath}
\usepackage{amsfonts}
\usepackage{amssymb}
\usepackage{amsthm}
\usepackage{graphicx}
\usepackage{caption}
\usepackage{xcolor}
\usepackage{float}
\usepackage[hidelinks,colorlinks=true,linkcolor=blue,citecolor=blue]{hyperref}
\usepackage{natbib}
\usepackage{verbatim}
\usepackage{mathtools}
\usepackage{tikz}
\usepackage{bbm}
\usepackage{subcaption}
\usepackage{authblk}
\theoremstyle{plain}
\newtheorem{theorem}{Theorem}
\newtheorem{proposition}{Proposition}
\newtheorem{lemma}{Lemma}
\newtheorem{corollary}{Corollary}

\newtheorem{definition}{Definition}

\newcommand{\reals}{\mathbb{R}}
\newcommand{\naturals}{\mathbb{N}}

\newcommand{\rationals}{\mathbb{Q}}

\newcommand{\Acal}{\mathcal{A}}

\newcommand{\Fcal}{\mathcal{F}}
\newcommand{\Gcal}{\mathcal{G}}

\newcommand{\Xcal}{\mathcal{X}}
\newcommand{\Ycal}{\mathcal{Y}}
\newcommand{\Vcal}{\mathcal{V}}
\newcommand{\Wcal}{\mathcal{W}}

\newcommand{\Ecal}{\mathcal{E}}

\DeclareMathOperator*{\expect}{\mathbb{E}}

\DeclarePairedDelimiter\floor{\lfloor}{\rfloor}

\newcommand{\indicator}{\mathbbm{1}}

\newcommand{\norm}[1]{\left\lVert#1\right\rVert}

\usepackage[letterpaper,top=2cm,bottom=2cm,left=3cm,right=3cm,marginparwidth=1.75cm]{geometry}

\begin{document}

\title{Is Zero-Shot Super-Resolution Possible in Operator Learning?}
\author[ ]{Unique Subedi \qquad \qquad  Ambuj Tewari}
\affil[ ]{Department of Statistics, University of Michigan}
\affil[ ]{\texttt{\{subedi, tewaria\}@umich.edu}}
\date{}

\maketitle

\begin{abstract}%
   Neural operators are often reported to exhibit zero-shot super-resolution, a phenomenon in which a model trained on coarse grids produces accurate predictions on finer testing grids without additional retraining. Despite strong empirical evidence, the theoretical foundations of this phenomenon remain unclear. In this work, we provide a systematic theoretical study of zero-shot super-resolution in operator learning. We first show that zero-shot super-resolution can be information-theoretically impossible even in benign settings such as when the input functions are available over the entire continuum and the ground truth is a simple rank-one linear operator. We then identify H{\" o}lder smoothness of the output functions as a sufficient condition for zero-shot super-resolution and derive corresponding generalization bounds. Finally, we also validate the identified failure modes through experimental results.
\end{abstract}

\section{Introduction}
Operator learning is a data-driven approach for learning complex nonlinear mappings between infinite-dimensional function spaces. These mappings often correspond to solution operators of partial differential equations (PDEs) that map problem specifications or initial conditions to solution functions. Neural operators are a class of neural-network-based operator models designed to learn such mappings \citep{kovachki2023neural}. They have demonstrated remarkable empirical success across a wide range of PDEs arising in practice, from quantum mechanics \citep{mizera2023scattering} to fluid dynamics \citep{wang2024recent}. In addition, a growing body of work has reported an intriguing capability of neural operators, commonly referred to as zero-shot super-resolution \citep{li2020fourier}. 

Although operator learning concerns mappings between functions defined on continuous domains, the learner often has access only to function values on a predefined discrete grid for computational feasibility. In this setting, zero-shot super-resolution refers to the phenomenon in which models trained using supervision on a coarse output grid exhibit good performance when evaluated on a substantially finer grid at test time, \emph{without any additional retraining or fine-tuning.}  Despite being frequently reported in empirical studies \citep{jiang2023efficient, yang2024fourier, sinha2025effectiveness}, zero-shot super-resolution remains poorly understood from a theoretical perspective. In this paper, we formally define the problem of zero-shot super-resolution and initiate a systematic theoretical investigation of zero-shot super-resolution in operator learning.

Before presenting our results, we emphasize two important points. First, the learner is always evaluated through its predicted
outputs; therefore, zero-shot super-resolution is fundamentally a property of generalization across
\emph{output resolutions}. Second, the  ``zero-shot" property requires that the learner does not
observe output functions at finer resolutions during training. However, the learner may still have
access to higher-resolution input functions, or even continuum inputs, during training. 
While the setting in which input functions are available at higher resolution than output functions
may appear uncommon, it is not unrealistic. In many applications, input functions are specified by
the practitioner and are therefore relatively inexpensive to obtain at high resolution. In contrast,
generating output functions often requires running expensive PDE solvers, making high-resolution
outputs significantly more costly. The ability of many operator-learning models to handle inputs at varying resolutions is commonly
referred to as discretization invariance. Although related, discretization invariance and zero-shot
super-resolution are distinct concepts (see Section \ref{sec:zeroshot-vs-dis-inv}).

A first natural question in our investigation is whether zero-shot super-resolution is possible at all. Our first main result is an impossibility theorem establishing that zero-shot super-resolution inference can fail even in extremely benign settings. We construct a class of rank-one operators for which learning is trivial when the training and testing grids coincide, yet for which no estimator, including empirically successful ones such as Fourier Neural Operators and DeepONets, achieves vanishing error when evaluated on a finer grid than the one used for training. We further support our theoretical results with numerical experiments that demonstrate the failure modes predicted by our lower bound (see Figure \ref{fig:residual_plot}). This lower bound complements recent empirical findings of \citep{gao2025discretization, sakarvadia2025false}, which provide evidence that zero-shot super-resolution can fail in certain regimes. Taken together, these results establish that zero-shot super-resolution is impossible without additional structural assumptions.

\begin{figure}[h]
    \centering
    \includegraphics[width=\linewidth]{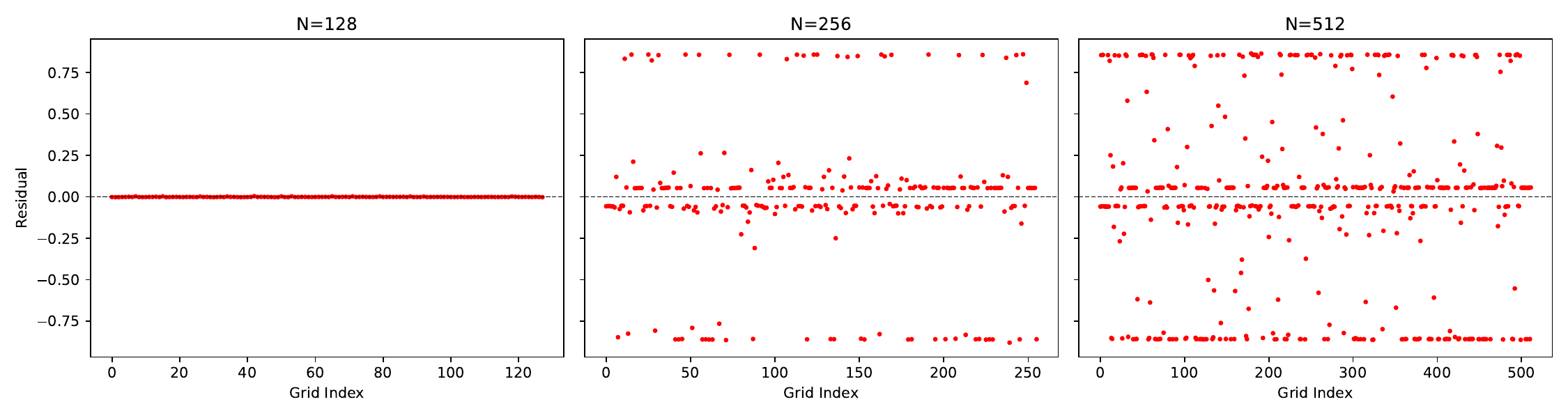}
    \caption{Residuals between the predicted output $\hat{w}$ and the ground truth $w$ at test resolutions $N = 128, 256, 512$, with training performed at $N = 128$. While residuals are negligible at the training resolution, errors grow substantially as the grid is refined, illustrating the failure of zero-shot super-resolution.}
    \label{fig:residual_plot}
\end{figure}

Given our lower bound, the next natural question is when zero-shot super-resolution is possible. We show that zero-shot super-resolution can indeed be achieved when the output functions produced by both the ground-truth operator and the learned operator are H{\" o}lder continuous. Under this assumption, we derive a zero-shot generalization bound that decomposes the test-time error on a finer grid into the error on the coarse grid and a term capturing extrapolation across resolutions. Finally, we connect our assumptions to classical PDE theory and modern neural-operator architectures. We show that H{\" o}lder regularity  arises naturally for solutions of many elliptic and parabolic PDEs via standard regularity results, and we verify that common neural-operator architectures produces continuous outputs under mild conditions on their kernels and biases.

\section{Related Works}\label{sec:related_works}
Zero-shot super-resolution in operator learning was first reported by \cite{li2020fourier} in the context of Fourier Neural Operators. Since then, the ability of neural operators to exhibit zero-shot super-resolution has been reported in a number of empirical studies across diverse applied settings \citep{jiang2023efficient, luo2024hierarchical, yang2024fourier, yasuda2025zero, sinha2025effectiveness}.

Such cross-resolution generalization is feasible in the first place because these models are discretization invariant, in the sense that they can be evaluated on grids with resolutions different from those used during training. These discretization-invariant operator-learning frameworks were developed in a series of works, with representative examples including \citep{bhattacharya2021model, li2020neural, li2020fourier, lu2021learning, nelsen2021random, bartolucci2023representation}. We refer the reader to the survey articles by \cite{boulle2023mathematical}, \cite{kovachki2024operator}, \cite{azizzadenesheli2024neural}, and \cite{subedi2025operator} for a more comprehensive overview.

Although zero-shot super-resolution has been widely reported in empirical studies, its failures have also been documented. For example, \cite{li2024physics} observed that Fourier Neural Operators fail to match the behavior of Kolmogorov flow in frequency regimes not present in the training data, and proposed physics-informed constraints as a corrective measure. Similarly, \cite{raonic2023convolutional} noted the inability of common operator-learning models such as Fourier Neural Operators, DeepONets, and Galerkin Transformers to achieve zero-shot super-resolution for transport equations, and hypothesized that this limitation arises from a lack of translation equivariance. They empirically showed that convolutional neural operators, which are translation equivariant, can generalize to unseen resolutions for these equations. Moreover, \cite{gao2025discretization} reported analogous failures of standard operator-learning methods in cross-resolution generalization and proposed techniques to mitigate them. \cite{gao2025discretization} also established an upper bound on the difference in prediction error between operators evaluated at two different resolutions and inferred, based on the upper bound, that the error can grow as the resolution gap increases. While such upper bounds are useful for building intuition, they do not conclusively rule out zero-shot super-resolution, as doing so requires a corresponding lower bound on the error.

A recent empirical study by \cite{sakarvadia2025false} systematically evaluated the methods proposed in \cite{li2024physics, raonic2023convolutional, gao2025discretization} and demonstrated that these approaches still fail to achieve reliable cross-resolution inference. Based on these findings, \cite{sakarvadia2025false} concluded that zero-shot super-resolution is fundamentally an out-of-distribution generalization problem and proposed multi-resolution training as a remedy. While this strategy can improve performance at higher resolutions, it breaks the zero-shot paradigm by allowing the model access to higher-resolution data during training.

\section{Problem Formulation}

Let $ \Xcal \subseteq \mathbb{R}^p $ and $ \Ycal \subseteq \mathbb{R}^d $ be bounded domains. Define $ L^2(\Xcal) := \{ f : \Xcal \to \mathbb{R} \mid \int_{\Xcal} |f(x)|^2\, dx < \infty \} $ as the space of square-integrable functions over $ \Xcal $. Let $ \Vcal \subseteq L^2(\Xcal) $ and $ \Wcal \subseteq L^2(\Ycal) $ be Banach spaces of functions, and suppose the ground truth operator of interest is $ \operatorname{G} : \Vcal \to \Wcal $. For example, $ \operatorname{G} $ might represent a PDE solution operator.

In the statistical learning setting, the learner observes $ n $ i.i.d.\ samples $ \{(v_i, w_i)\}_{i=1}^n $, where $ v_i \sim \mu $ and $ w_i = \operatorname{G}(v_i) $ (or a noisy version thereof). Crucially, we assume that outputs $ w_i $ are only available on a discrete output grid $ \Ycal_{\mathrm{train}} \subset \Ycal $, with $ |\Ycal_{\mathrm{train}}| < \infty $. For now, we assume that the inputs $ v_i $ are accessible over the entire continuous domain $ \Xcal $. The case where inputs are observed only on a discrete grid $ \Xcal_{\mathrm{train}} $ is discussed in Section~\ref{sec:discreteX}, with additional related remarks provided in Section~\ref{sec:zeroshot-vs-dis-inv}. Using the observed data, the learner constructs an estimator $ \widehat{\operatorname{F}}_n $. For simplicity, we use the same notation to refer both to the learned operator and to the learning rule. Although $ \widehat{\operatorname{F}}_n $ can in principle be any map $ \Vcal \to \Wcal $, in practice it is usually chosen from a restricted function class $ \Fcal \subseteq \Wcal^{\Vcal} $, such as neural operators \citep{kovachki2023neural}. Note that the ground truth operator $ \operatorname{G} $ may not belong to $ \Fcal $.

At the very least, the learner's goal is to minimize the expected error on the training grid defined as
\[
\Ecal_{\mu}(\widehat{\operatorname{F}}_n, \operatorname{G}, \Ycal_{\mathrm{train}}) := \expect_{v \sim \mu} \left[ \frac{1}{|\Ycal_{\mathrm{train}}|} \sum_{y \in \Ycal_{\mathrm{train}}} \left( \widehat{\operatorname{F}}_n(v)(y) - \operatorname{G}(v)(y) \right)^2 \right].
\]
However, often at test time, the output of the estimator is evaluated on a finer grid $ \Ycal_{\mathrm{test}} \supseteq \Ycal_{\mathrm{train}} $, where $ |\Ycal_{\mathrm{test}}| < \infty $. The test-time error is defined as
\[
\mathcal{E}_{\mu}(\widehat{\operatorname{F}}_n, \operatorname{G}, \Ycal_{\mathrm{test}}) := \expect_{v \sim \mu} \left[ \frac{1}{|\Ycal_{\mathrm{test}}|} \sum_{y \in \Ycal_{\mathrm{test}}} \left( \widehat{\operatorname{F}}_n(v)(y) - \operatorname{G}(v)(y) \right)^2 \right].
\]

We emphasize that the estimator is trained using supervision only on the coarse grid
$\Ycal_{\mathrm{train}}$, yet is expected to generalize to the finer grid
$\Ycal_{\mathrm{test}}$. This is precisely the main goal of zero-shot super-resolution. Accordingly, we seek conditions under which a small expected error
$\mathcal{E}_{\mu}(\widehat{\operatorname{F}}_n, \operatorname{G}, \Ycal_{\mathrm{train}})$ on the
training grid guarantees a correspondingly small expected error on the finer testing grid
$\mathcal{E}_{\mu}(\widehat{\operatorname{F}}_n, \operatorname{G}, \Ycal_{\mathrm{test}})$, even
though the learning algorithm has never observed outputs on
$\Ycal_{\mathrm{test}} \setminus \Ycal_{\mathrm{train}}$. Formally, we ask whether there exists a
condition such that, for every $\varepsilon > 0$, if there exists $n$ sufficiently large for which
\[
\mathcal{E}_{\mu}(\widehat{\operatorname{F}}_n, \operatorname{G}, \Ycal_{\mathrm{train}})
\le \varepsilon,
\]
then
\[
\mathcal{E}_{\mu}(\widehat{\operatorname{F}}_n, \operatorname{G}, \Ycal_{\mathrm{test}})
\le c\,\varepsilon
\]
for some universal constant $c > 0$. Ideally, we want the aforementioned condition to be distribution-free. If such condition exists, then we say that zero-shot super-resolution is possible for the tuple $(\widehat{\operatorname{F}}_n, \operatorname{G}, \Ycal_{\mathrm{train}}, \Ycal_{\mathrm{test}})$.

\subsection{Zero-Shot Super-Resolution vs Discretization Invariance}\label{sec:zeroshot-vs-dis-inv}
In the operator learning community, the concept of zero-shot super-resolution is sometimes conflated with a related notion of discretization invariance, as formalized in \cite[Definition 4]{kovachki2023neural}. An operator $ \operatorname{F} $ is said to be discretization invariant with respect to a sequence of nested input grids $ \{\Xcal_k\}_{k \in \mathbb{N}} $, where $ \Xcal_k \subseteq \Xcal_{k+1} $ and $ \lim_{k \to \infty} \Xcal_k = \Xcal $, if
\[
\lim_{k \to \infty} \sup_{v \in \Acal} \left\| \operatorname{F}(v_{|\Xcal_k}) - \operatorname{F}(v) \right\|_{L^2(\Ycal)} = 0
\]
for every compact subset $ \Acal \subseteq \Vcal $, where $ v_{|\Xcal_k} $ denotes the restriction of $ v $ to the grid $ \Xcal_k $. This definition presupposes a meaningful way to evaluate $ \operatorname{F}(v_{|\Xcal_k}) $.

Importantly, discretization invariance focuses solely on input resolution. It ensures that as the input grid becomes increasingly refined, the operator’s output converges to its continuum evaluation. However, zero-shot super-resolution concerns a different question: how well an operator trained on a coarse \emph{output} grid generalizes to a finer one. Since discretization invariance is defined using the continuum norm in the output space, it provides no information about generalization across different \emph{output grids}.

To clearly separate these two notions, we begin by assuming that input functions $ v $ are available over the entire domain $ \Xcal $. This removes the issue of input discretization, allowing us to focus purely on output resolution. In fact, discretization invariance is trivially satisfied in this setting, as $ \Xcal_k = \Xcal $ for all $ k $. Despite this, our impossibility result in Section~\ref{sec:impossibility} shows that zero-shot super-resolution remains a fundamentally nontrivial problem, even when discretization invariance holds.

\section{Impossibility of Zero-Shot Super-Resolution: A Lower Bound}\label{sec:impossibility}

We now show that generalization to finer output grids cannot always be guaranteed, even for simple operator learning problems. Specifically, we construct a setting where learning is trivial when the training and testing grids coincide ($ \Ycal_{\mathrm{train}} = \Ycal_{\mathrm{test}} $), but impossible when the model is trained on a coarse grid and evaluated on a finer one.

Let $ \mathcal{X} = \mathcal{Y} = [0,1) $, where the exclusion of the endpoint is simply so that our construction allows periodic boundary conditions. Let $ N_1, N_2 \in \mathbb{N} $ with $ N_2 > N_1 $. The learner observes outputs $ w_i = \operatorname{G}(v_i) $ only on the coarse, uniformly spaced grid of size $N_1$ but evaluated on the finer grid of size $N_2$. That is,
\[
\mathcal{Y}_{\text{train}} := \left\{ 0, \tfrac{1}{N_1}, \ldots, \tfrac{N_1 - 1}{N_1} \right\} \quad \text{ and } \quad  \mathcal{Y}_{\text{test}} := \left\{ 0, \tfrac{1}{N_2}, \ldots, \tfrac{N_2 - 1}{N_2} \right\}.
\]
To ensure that $ \mathcal{Y}_{\text{train}} \subseteq \mathcal{Y}_{\text{test}} $, we assume $ N_2 = m N_1 $ for some $ m \in \mathbb{N} $, as done in prior works. For example,  \cite{li2020fourier} uses $ m = 2^k $ for $ k \geq 1 $.

We consider a learning problem where the ground truth operator $ \operatorname{G} $ belongs to a known class $ \Gcal $, though the learner does not know which specific operator in the class is realized. This reflects practical scenarios where structural knowledge about the PDE allows the learner to constrain the solution operator to a known class, even if the specific instance depends on problem-specific parameters. The result below shows that learning $ \operatorname{G} $ is straightforward when the training and testing grids match, but zero-shot super-resolution is impossible when the evaluation grid is finer than the training grid.

\begin{theorem}[Impossibility of Zero-Shot Super-Resolution]
\label{thm:lb-general}
Let $\Xcal = \Ycal = [0,1)$, and let $\Gcal$ be a known class of rank-one linear operators from
$L^2(\Xcal)$ to $L^2(\Ycal)$ such that $\|\operatorname{G}\|_{\mathrm{op}} \le 1$ for all
$\operatorname{G} \in \Gcal$. Let $\mu$ be a probability measure supported on $L^2(\Xcal)$ such
that every $v \sim \mu$ satisfies
\[
a \le |v(x)| \le 1 \quad \text{for all } x \in \Xcal,
\]
for some constant $a > 0$. Then there exists a ground-truth operator $\operatorname{G} \in \Gcal$ such that, for any $ N_1, N_2 \in \mathbb{N} $ with $ N_2 = mN_1 $, the following
statements hold.

\begin{itemize}
\item[\emph{(i)}] \emph{(Perfect learning on the training grid)} 
There exists an estimator
$\widehat{\operatorname{F}}_1' : L^2(\Xcal) \to L^2(\Ycal)$, depending only on a single sample
$(v_1, w_1)$ with $v_1 \sim \mu$ and the output $w_1 = \operatorname{G}(v_1)$ observed only on
$\Ycal_{\mathrm{train}}$, such that
\[
\expect_{v \sim \mu}
\left[
\frac{1}{|\Ycal_{\mathrm{train}}|}
\sum_{y \in \Ycal_{\mathrm{train}}}
\bigl( \widehat{\operatorname{F}}_1'(v)(y) - \operatorname{G}(v)(y) \bigr)^2
\right]
= 0.
\]

\item[\emph{(ii)}] \emph{(Failure of zero-shot super-resolution)} 
For every estimator $\widehat{\operatorname{F}}_n : L^2(\Xcal) \to L^2(\Ycal)$ trained on samples
$\{(v_i, w_i)\}_{i=1}^n$, where $v_i \sim \mu$ and each output $w_i = \operatorname{G}(v_i)$ is
observed only on the training grid $\Ycal_{\mathrm{train}}$, the expected test error on the finer
grid satisfies
\[
\expect_{v \sim \mu}
\left[
\frac{1}{|\Ycal_{\mathrm{test}}|}
\sum_{y \in \Ycal_{\mathrm{test}}}
\bigl( \widehat{\operatorname{F}}_n(v)(y) - \operatorname{G}(v)(y) \bigr)^2
\right]
\;\ge\;
\frac{a^2}{4}\left(1 - \frac{1}{m}\right).
\]
\end{itemize}
\end{theorem}

Part (i) shows that the problem is trivially learnable when the operator is evaluated on the same grid it was trained on. However, the lower bound in part (ii) remains non-zero for any $ m > 1 $ even with infinitely many samples ($n\to \infty$), implying that zero-shot super-resolution is impossible in this setting. Importantly, our lower bound does not rely on any structural assumptions about the estimator $ \widehat{\operatorname{F}}_n $. In particular, the estimator need not be linear, unlike the ground truth operator $ \operatorname{G} $. As a result, the bound applies even to highly expressive nonlinear models, including neural operators such as Fourier Neural Operators (FNOs). 

The proof of Theorem \ref{thm:lb-general} is provided in Appendix \ref{appdx:proof_lb}.  We briefly outline the proof strategy. The construction defines a family of
rank-one operators that agree on the coarse training grid but may differ
arbitrarily on points in $\Ycal_{\mathrm{test}}\setminus \Ycal_{\mathrm{train}}$.
Because the outputs are observed only on $\Ycal_{\mathrm{train}}$, a single sample
is sufficient to identify the operator on the training grid, yielding zero
error in part (i). However, the observations provide no information about the
operator values at unseen test points. Thus, there exist multiple
operators that are indistinguishable from the training data yet disagree on
$\Ycal_{\mathrm{test}}\setminus \Ycal_{\mathrm{train}}$. A minimax argument then shows
that any estimator must incur a non-vanishing error on the finer grid,
establishing the lower bound in part (ii). 

The lower bound highlights that the fundamental obstacle is the absence of
regularity in the output functions. Without continuity or smoothness
assumptions, the values of the output function between observed grid points
can vary arbitrarily, making extrapolation to unseen locations impossible.
This observation motivates the regularity assumptions introduced in the next
section.

\section{A Generalization Bound for Zero-Shot Super-Resolution}
\label{sec:upper}


Given Theorem~\ref{thm:lb-general}, a natural next question is whether zero-shot super-resolution becomes possible under additional regularity assumptions. We show that this is indeed the case under suitable continuity assumptions on the output functions.
Intuitively, without continuity, the output functions can vary arbitrarily between observed grid points, making any attempt to infer their values at unseen locations fundamentally ill-posed. This observation motivates the question of whether continuity of the output functions is sufficient to guarantee zero-shot super-resolution. While continuity is indeed sufficient, it is a qualitative property and does not yield quantitative error bounds. To obtain meaningful rates, we therefore impose a stronger regularity assumption. Specifically, we assume that the output functions are uniformly Hölder continuous and derive quantitative generalization bounds for zero-shot super-resolution under this assumption.

Recall that a function $ w: \Ycal \to \reals $ is said to be uniformly Hölder continuous with exponent $ \alpha \in (0,1] $ on $ \Ycal $ if
\[
[w]_{\alpha} := \sup_{\substack{y_1, y_2 \in \Ycal \\ y_1 \neq y_2}} \frac{|w(y_1) - w(y_2)|}{|y_1 - y_2|_2^{\alpha}} < \infty.
\]
Let $ C^{0,\alpha}(\Ycal) $ denote the space of real-valued functions on $ \Ycal $ with finite H{\" o}lder exponent. In this section, we consider the case where the output space $ \Wcal $ is constrained to functions that are uniformly Hölder continuous for some constant $ c > 0 $
\[
\Wcal \subseteq \left\{ w \in C^{0,\alpha}(\Ycal) \,:\, [w]_{\alpha} \leq c \right\}.
\]
 This regularity assumption allows us the learner to extrapolate between grid points in the output domain $\Ycal$.

In addition to regularity assumptions on the output functions, it is also necessary to ensure that the training and testing grids are sufficiently close for meaningful extrapolation to be possible. For example, consider a setting in which the training grid is the uniform discretization of $[0,1)$ with $N$ points, $
\Ycal_{\mathrm{train}} = \left\{ 0, \tfrac{1}{N}, \ldots, \tfrac{N-1}{N} \right\},$
while the testing grid is given by $
\Ycal_{\mathrm{test}} = \Ycal_{\mathrm{train}} \cup (\Ycal_{\mathrm{train}} + 5),$
that is, the union of the same uniform grid on $[0,1)$ and a shifted copy on $[5,6)$. Although $\Ycal_{\mathrm{train}} \subseteq \Ycal_{\mathrm{test}}$, it is unreasonable to expect any meaningful extrapolation to the additional test points in $[5,6)$ based solely on observations from $[0,1)$. Such pathological cases typically do not arise in existing empirical studies, which usually assume that both $\Ycal_{\mathrm{train}}$ and $\Ycal_{\mathrm{test}}$ are uniform discretizations of the same domain, with the testing grid being a refinement of the training grid. However, since our goal is to establish results at a higher level of generality, potentially accommodating non-uniform or even arbitrary grids, we introduce quantitative notions of similarity between the training and testing grids that are sufficient to allow meaningful extrapolation.

\begin{definition}\label{def:grid_params}
Let $ \Ycal_{\mathrm{train}}, \Ycal_{\mathrm{test}} \subset \Ycal \subseteq \mathbb{R}^d $ be finite sets.

\begin{itemize}
  \item[\emph{(i)}] The \emph{coverage} between the training and testing grids is defined as
  \[
  \beta := \max_{y \in \Ycal_{\mathrm{test}}} \min_{y' \in \Ycal_{\mathrm{train}}} |y - y'|_2.
  \]

  \item[\emph{(ii)}] The \emph{load-balancing factor} is defined as
  \[
  \nu := \max_{z \in \Ycal_{\mathrm{train}}} \sum_{y \in \Ycal_{\mathrm{test}}} \mathbbm{1} \left\{ \mathrm{nn}(y) = z \right\},
  \]
  where $ \mathrm{nn}(y) \in \Ycal_{\mathrm{train}} $ denotes the nearest neighbor of $ y $, with deterministic tie-breaking chosen to minimize $ \nu $.
\end{itemize}
\end{definition}

The quantity $\beta$ is the one-sided Hausdorff distance from the testing grid to the training grid and quantifies how well the testing domain is covered by the training grid. The load-balancing factor $\nu$  quantifies the maximum number of testing points whose predictions rely on a single training
point. Intuitively, larger values of $\nu$ place a greater burden on individual training points to represent multiple unseen testing locations, increasing the risk of extrapolation error.

Given these quantities, the following result establishes a bound on the zero-shot super-resolution generalization error.

\begin{theorem}[Zero-Shot Super-Resolution Generalization Bound]
\label{thm:ub}
Let $ \operatorname{G} $ denote the ground truth operator, and let $ \widehat{\operatorname{F}}_n $ be any estimator trained on $ n $ samples $ \{(v_i, w_i)\}_{i=1}^n $, where each $ w_i = \operatorname{G}(v_i) $ is observed only on a discrete training grid $ \Ycal_{\emph{train}} \subseteq \Ycal $.  Assume that there exist $ \alpha \in (0,1] $ and $ c > 0 $ such that for all $ v $, both $ \operatorname{G}(v) $ and $ \widehat{\operatorname{F}}_n(v) $ belong to the Hölder class $ C^{0,\alpha}(\Ycal) $ with seminorm at most $ c $. Let $ \Ycal_{\mathrm{test}} \supseteq \Ycal_{\mathrm{train}} $ be a test grid, and let $\beta$ and $\nu$ be as defined in Definition~\ref{def:grid_params}. Then,
\begin{equation*}
\begin{split}
\mathcal{E}_{\mu}(\widehat{\operatorname{F}}_n, \operatorname{G}, \Ycal_{\emph{test}}) 
\leq 
\frac{|\Ycal_{\emph{train}}|}{|\Ycal_{\emph{test}}|} \cdot (2\nu-1)\cdot 
\mathcal{E}_{\mu}(\widehat{\operatorname{F}}_n, \operatorname{G}, \Ycal_{\emph{train}}) + \left(1-\frac{|\Ycal_{\mathrm{train}}|}{|\Ycal_{\mathrm{test}}|}\right) \cdot 8c^2 \cdot \beta^{2\alpha}.
\end{split}
\end{equation*}
\end{theorem}

This bound decomposes the expected error on the higher-resolution test grid into two components: a statistical term and an extrapolation term. The first term corresponds to the statistical error of a finite-dimensional vector-valued regression problem. Since the outputs are observed only on the discrete grid $\Ycal_{\mathrm{train}}$, the learning task reduces to estimating a map into $\mathbb{R}^{|\Ycal_{\mathrm{train}}|}$. Consequently, this term can be analyzed using standard tools from statistical learning theory, including Rademacher complexity and covering number bounds. The second term captures the error due to extrapolation to unseen grid points, which is the primary quantity of interest for us. We therefore concentrate on this term for the remainder of the analysis.

The second term captures the error incurred when extrapolating from the training grid to unseen points on the test grid. In the special case where $\Ycal_{\mathrm{train}} = \Ycal_{\mathrm{test}}$, there is no extrapolation. In this case, the second term vanishes and $\nu = 1$, so the bound reduces to $
\mathcal{E}_{\mu}(\widehat{\operatorname{F}}_n, \operatorname{G}, \Ycal_{\mathrm{train}}),$
which recovers the expected error on the training grid. More generally, suppose that 
\[
\mathcal{E}_{\mu}(\widehat{\operatorname{F}}_n, \operatorname{G}, \Ycal_{\mathrm{train}}) \to 0
\quad \text{as } n \to \infty.
\]
Then, for any $\varepsilon > 0$, there exist $n_0 \in \naturals$ and $\beta_0 > 0$ such that for all $n \ge n_0$ and all grids satisfying $\beta \le \beta_0$,
\[
\mathcal{E}_{\mu}(\widehat{\operatorname{F}}_n, \operatorname{G}, \Ycal_{\mathrm{test}}) \le \varepsilon.
\]
This shows that zero-shot super-resolution is achievable with sufficiently large sample size and sufficiently fine training grid.

The full proof of Theorem~\ref{thm:ub} is deferred to Appendix~\ref{appdx:ub}. We briefly outline
the main idea. For each test point, we compare its prediction error to that
of its nearest neighbor in the training grid. This yields a decomposition
into a training-grid error term and an extrapolation term. The main technical
challenge is that the training-grid error is controlled only on average,
whereas the nearest-neighbor argument requires relating this average error to
errors at individual training points. The load-balancing factor $\nu$
quantifies how many testing points may depend on the same training point,
while H\"older continuity controls the extrapolation error through the
coverage parameter $\beta$. Combining these ingredients yields the stated
bound.


\vspace{0.3cm}
\noindent \textbf{On the necessity of $\nu$.} The quantity $\nu$ is introduced to rule out pathological scenarios in which a small number of
training points are responsible for extrapolating to a large number of testing points.
To illustrate this, consider the case where $\Ycal = [0,1]$ and the training grid is $
\Ycal_{\mathrm{train}} = \left\{ 0, \tfrac{1}{N}, \tfrac{2}{N}, \ldots, \tfrac{1}{2}, 1 \right\}$. Note that there are no other points between $1/2$ and $1$.
Suppose the testing grid satisfies $
\Ycal_{\mathrm{test}} \subseteq \Ycal_{\mathrm{train}} \cup \Ycal_{\mathrm{new}},$
where the set of new testing points $\Ycal_{\mathrm{new}} \subseteq [1-\varepsilon, 1]$ lies near
the endpoint for some small $\varepsilon > 0$. In this setting, the coverage $\beta$ is at most $\max\{\varepsilon, 1/N\}$, \emph{a good geometric
coverage}. However, suppose the estimator performs well at all points in $\Ycal_{\mathrm{train}}$ except at
the endpoint $y = 1$. When $N$ is large, the error at this single point contributes negligibly to
the average error on the training grid. Yet all points in $\Ycal_{\mathrm{new}}$ have $y = 1$ as their nearest
neighbor in the training grid, so extrapolation to the unseen region depends almost entirely on
this poorly estimated point. In such cases, reliable extrapolation is unrealistic. The quantity
$\nu$ captures this phenomenon by measuring how many testing points rely on each training point.
In the example above, $\nu = |\Ycal_{\mathrm{new}}|$, which can be arbitrarily large depending on
the number of unseen testing points. This illustrates why controlling $\nu$ is necessary for
meaningfully bounding the extrapolation error.

\subsection{Specializing Theorem~\ref{thm:ub} to Uniform Grids}

While Theorem~\ref{thm:ub} is stated under general assumptions, it is helpful to study its implications in more a concrete setting. We therefore specialize the bound to uniform grids over $[0,1)^d$. The proof is deferred to Appendix~\ref{appdx:proof_corr}.

\begin{corollary}
\label{cor:uniform} 
Under the assumptions of Theorem \ref{thm:ub}, suppose $ \Ycal_{\emph{train}} $ and $ \Ycal_{\emph{test}} $ are uniform grids over $ [0,1)^d $ with resolutions $ N_1 $ and $ N_2 $ respectively along each direction, such that $N_2 = m N_1$ for some $m \in \naturals$.
Then, the expected test error satisfies
\[
\mathcal{E}_{\mu}(\widehat{\operatorname{F}}_n, \operatorname{G}, \Ycal_{\emph{test}}) \leq 2^{d+1}\, \mathcal{E}_\mu(\widehat{\operatorname{F}}_n, \operatorname{G}, \Ycal_{\mathrm{train}}) + \left( 1-\frac{1}{m^d}\right) \cdot 8c^2\cdot  \left(\frac{\sqrt{d}}{N_1} \right)^{2\alpha}.
\]
\end{corollary}

\noindent Corollary~\ref{cor:uniform} makes the dependence on the grid resolution explicit. 
The extrapolation error decays as $N_1^{-2\alpha}$, showing that finer training grids improve zero-shot super-resolution. 
Crucially, this term depends only on the training resolution $N_1$ and is independent of the test grid refinement factor $m$.
In particular, the error does not degrade as the test resolution increases, indicating that the training resolution is the fundamental bottleneck for zero-shot super-resolution.

The $2^{d+1}$ dependence in Corollary~\ref{cor:uniform} is somewhat loose, as it arises from applying
the general bound in Theorem~\ref{thm:ub}, which does not exploit the additional regularity of uniform grids. For
uniform grids, the analysis can be refined to obtain the sharper bound
\[
\mathcal{E}_\mu(\widehat{\operatorname{F}}_n, \operatorname{G}, \Ycal_{\mathrm{test}})
\le
2\, \mathcal{E}_\mu(\widehat{\operatorname{F}}_n, \operatorname{G}, \Ycal_{\mathrm{train}})
+
\left( 1 - \frac{1}{m^d} \right)
8c^2 \left( \frac{\sqrt{d}}{N_1} \right)^{2\alpha}.
\]
This refinement is obtained simply by re-deriving the proof of Theorem \ref{thm:ub} under the additional
structure imposed by uniform grids. Since the argument closely parallels the general case, we
provide only a proof sketch in Appendix~\ref{appdx:refined}. Observing that $1 - 1/m^d \le 1$, the
bound further simplifies to
\[
\mathcal{E}_\mu(\widehat{\operatorname{F}}_n, \operatorname{G}, \Ycal_{\mathrm{test}})
\le
2\, \mathcal{E}_\mu(\widehat{\operatorname{F}}_n, \operatorname{G}, \Ycal_{\mathrm{train}})
+
8c^2 d^{\alpha} N_1^{-2\alpha}.
\]

\subsection{Non-uniform Grids}

In many applications, the grids used in practice are not exactly uniform. In full generality, deriving meaningful generalization bounds for arbitrary non-uniform grids is challenging, as cross-resolution behavior depends on the relative geometry of the training and testing grids. We therefore introduce a simple condition that directly controls this joint property. 

Consider a nested sequence of grids
\[
\Ycal_1 \subseteq \Ycal_2 \subseteq \cdots \subseteq \Ycal_k \subseteq \cdots \subseteq \Ycal,
\]
where $\Ycal \subseteq \mathbb{R}^d$ is the underlying domain. Define the fill distance
\[
h_k := \sup_{y \in \Ycal} \min_{z \in \Ycal_k} |y - z|_2
\]
and the separation distance
\[
q_k := \tfrac{1}{2} \min_{\substack{z, z' \in \Ycal_k \\ z \neq z'}} |z - z'|_2.
\]

Let $\Ycal_k$ denote the training grid and $\Ycal_\ell$ the testing grid. For nested grids, the coverage parameter satisfies
\[
\beta_{k,\ell} \le h_k.
\]
Moreover, a packing argument (see Appendix \ref{appdx:nonuniform}) yields
\[
\nu_{k,\ell}
\le
\left(1+\frac{h_k}{q_\ell}\right)^d.
\]

Consequently, if there exists a constant $C > 0$ such that
\[
\frac{h_k}{q_\ell} \le C,
\]
then the load-balancing factor $\nu_{k,\ell}$ is uniformly bounded. In this case, Theorem~\ref{thm:ub} implies
\[
\mathcal{E}_\mu(\widehat{\operatorname{F}}_n, \operatorname{G}, \Ycal_\ell)
\;\lesssim\;
\mathcal{E}_\mu(\widehat{\operatorname{F}}_n, \operatorname{G}, \Ycal_k)
+
h_k^{2\alpha}
\qquad \text{for all } \ell \ge k.
\]

Thus, controlling the ratio $h_k/q_\ell$ is sufficient to ensure that both the coverage and load-balancing parameters remain well-behaved, leading to stable zero-shot super-resolution.

Since the result is expressed in terms of the geometric quantities $h_k$ and $q_\ell$, one might wonder whether this is merely a reformulation of the $(\beta, \nu)$ abstraction. However, there is an important distinction. The parameters $\beta$ and $\nu$ are properties of the pair $(\Ycal_{\text{train}}, \Ycal_{\text{test}})$ and explicitly capture the interaction between the two grids. In contrast, $h_k$ and $q_\ell$ are intrinsic properties of the individual grids themselves and can be analyzed independently. Moreover, these quantities are classical objects in approximation theory and scattered data analysis. In particular, grids satisfying $
q_k \le h_k \le B q_k$
for some constant $B > 0$ are known as quasi-uniform grids \citep[Chapter4]{wendland2004scattered}, with the special case $B = 1$ corresponding to uniform grids.

\section{On the Assumption of $\Wcal \subseteq C^{0, \alpha}(\Ycal)$}
In this section, we examine when the assumption that the output functions are H{\" o}lder continuous is
reasonable. We first draw on classical PDE regularity theory to justify settings in which the
ground-truth functions are H{\" o}lder smooth. We then establish conditions under which common
operator-learning models produce H{\" o}lder-continuous predictions.

\subsection{H{\" o}lder Smoothness of Ground Truth Output Functions}
Consider a linear PDE of the form
\[
\operatorname{L} u = f,
\]
where $ u, f : \Omega \to \mathbb{R} $, $ \Omega \subseteq \mathbb{R}^d $ is a bounded domain, and $ u $ satisfies homogeneous boundary conditions. The goal is to learn the solution operator $ \operatorname{G} $, which maps the input $ f $ (typically representing system specifications) to the corresponding solution $ u $. Since $ \operatorname{G} $ is  a partial inverse to the differential operator $ \operatorname{L} $, it can be expressed as an integral operator. In particular,
\[
u(y) = (\operatorname{G}f)(y) = \int_{\Omega} g(y,x)\, f(x)\, dx,
\]
where $ g : \Omega \times \Omega \to \mathbb{R} $ is the Green’s function associated with $ \operatorname{L} $ \citep[Chapter 1]{hartmann2012green}. The following result states that the Hölder continuity of the Green's functions is sufficient to ensure that the solution $\operatorname{G}(f)$ is also Hölder continuous. 

\begin{proposition}\label{prop:green}
Suppose there exists $ c > 0 $ such that for almost every $ x \in \Omega $, the function $ y \mapsto g(y,x) $ belongs to $ C^{0,\alpha}(\Omega) $ with Hölder constant at most $ c $. Then, for every $ v \in L^2(\Omega) $, the output $ \operatorname{G}(v) $ lies in $ C^{0,\alpha}(\Omega) $.

\end{proposition}
\noindent The proof of Proposition \ref{prop:green} is deferred to Appendix \ref{appdx:proof_green}.   

Next, let us consider the case where the operator of interest is non-linear. To that end, a widely used benchmark PDE in operator learning is the equation
\[
(\operatorname{L}u)(x) := -\sum_{i=1}^d \sum_{j=1}^d \partial_i\big(a^{ij}(x) \, \partial_j u(x)\big),
\]
referred to as the Darcy flow equation in \cite{li2020fourier}. When the goal is to learn the mapping $ \operatorname{G}: f \mapsto u $, the solution can be expressed via an integral operator using the Green’s function. However, \cite{li2020fourier} instead considers the problem of learning the operator $ \operatorname{G}': a \mapsto u $. While $ \operatorname{G}' $ is not a linear operator anymore, regularity theory still provides useful information about $\operatorname{G}^{\prime}(a)$ when the coefficients $ a^{ij}(x) $ satisfy the uniform ellipticity condition. Precisely, if there exists $ \lambda > 0 $ such that
\[
\sum_{i,j=1}^d a^{ij}(x)\, \xi_i \xi_j \geq \lambda |\xi|^2 \quad \text{for all } \xi \in \mathbb{R}^d \text{ and a.e. } x \in \Omega,
\]
and if $ \|a^{ij}\|_{L^\infty} \leq \Lambda $ and $ f \in L^\infty(\Omega) $, then the classical De Giorgi–Nash–Moser theorem guarantees that the solution $ u $ is H{\" o}lder continuous in the interior of $ \Omega $. Thus, $ \operatorname{G}'(a) \in C^{0,\alpha}(\Omega) $, with the Hölder exponent depending only on $ \lambda $, $ \Lambda $, $ d $, and $ \|f\|_{L^\infty} $. For more details and a precise statement of the result, see \citep[Section 8.9]{gilbarg1977elliptic}.

The regularity of PDE solutions is a well-studied topic and typically depends on the specific structure of the equation. Since this is covered extensively in the literature, we refer the readers to standard references such the book by~\cite{evans2022partial}.

\vspace{0.3cm}

\noindent \textbf{H\"older vs.\ Sobolev Regularity.}
Much of PDE regularity theory is formulated in terms of weak solutions. That is, the solution $u$ may not be classically differentiable, but instead belongs to a Sobolev space $W^{k,p}(\Omega)$, meaning that all weak derivatives up to order $k$ lie in $L^p(\Omega)$.  When $k = 1$ and $p > d$, Morrey’s inequality \citep[Section~5.6.2]{evans2022partial} implies that such functions are H\"older continuous with exponent $\alpha = 1 - \frac{d}{p}$. More generally, Sobolev embedding theorems \citep[Section~5.6.3]{evans2022partial} ensure that $u \in C^{0,\alpha}(\Omega)$ for some $\alpha > 0$ whenever $k > d/p$. Thus, H\"older regularity assumptions can often be derived directly from Sobolev regularity via embedding results.

\subsection{H{\" o}lder Smoothness of Predicted Output Functions}

We next discuss conditions under which common operator-learning models produce
Hölder-continuous predictions.

\subsubsection{Linear Operators}

Consider a linear integral operator of the form
\[
v \mapsto \operatorname{F}(v),
\qquad
\text{where}
\qquad
\operatorname{F}(v)(y) = \int_{\Xcal} k(y,x)\, v(x)\, dx.
\]

\begin{proposition}
\label{prop:linear}
If $y \mapsto k(y,x) \in C^{0,\alpha}(\Ycal)$ for almost every $x \in \Xcal$ with Hölder constants
uniformly bounded in $x$, then $\operatorname{F}(v) \in C^{0,\alpha}(\Ycal)$ for every
$v \in L^2(\Xcal)$.
\end{proposition}

\noindent
The proof is identical to that of Proposition \ref{prop:green} and is therefore deferred to
Appendix \ref{appdx:proof_green}.

\subsubsection{Neural Operators}

Let $\Xcal = \Ycal$. Then, for a given input function $v$, a single layer of a neural operator is defined as 
\[ \operatorname{N}(v)(y) = \sigma \Big(\int_{\mathcal{X}} k(y, x)\, v(x) \,dx + b(y) \Big) \]
Here,  $b: \mathcal{Y} \to \mathbb{R}$ is a bias function and $\sigma: \mathbb{R} \to \mathbb{R}$ is the relu activation function. These layers are composed sequentially to get a multilayer neural operator.  The following result establishes that, given an squared integrable function as input, each layer of neural operator produces H{\" o}lder smooth functions.

\begin{proposition}\label{prop:holder-NO}
Let $ b \in C^{0,\alpha}(\Ycal) $, and suppose that for each $ x \in \Xcal $, the function $ y \mapsto k(y, x) $ belongs to $ C^{0,\alpha}(\Ycal) $, with a Hölder coefficient uniformly bounded in $ x $. Then for every $ v \in L^2(\Ycal) $, we have $ \operatorname{N}(v) \in C^{0,\alpha}(\Ycal) $.

\end{proposition}
\noindent We defer the proof of Proposition \ref{prop:holder-NO} to Appendix \ref{appdx:proof_prop_NO}.

Next, we show that a multilayer neural operator $v \mapsto \operatorname{F}(v)$, defined by
\[
\operatorname{F}(v)
=
\operatorname{N}_L \circ \operatorname{N}_{L-1} \circ \cdots \circ \operatorname{N}_1(v),
\]
also produces H{\" o}lder-continuous outputs. If each layer $\operatorname{N}_t$ is parameterized by
$(k_t, b_t)$ satisfying the assumptions of Proposition~\ref{prop:holder-NO}, then repeated
application of Proposition~\ref{prop:holder-NO} implies that
$\operatorname{F}(v) \in C^{0,\alpha}(\Ycal)$. To verify that these assumptions apply at every layer, it suffices to check that the output of each
intermediate layer lies in $L^2(\Ycal)$. Notice that the output of the first layer
is $v_1 := \operatorname{N}_1(v)$, which belongs to $C^{0,\alpha}(\Ycal)$ by
Proposition \ref{prop:holder-NO}. Since $\Ycal$ is bounded, we have
\[
\|v_1\|_{L^2(\Ycal)}^2
=
\int_{\Ycal} |v_1(y)|^2\, dy
\le
\sup_{y \in \Ycal} |v_1(y)|^2 \cdot \mathrm{vol}(\Ycal)
<
\infty.
\]
The final step uses the fact that a continuous function on a bounded domain is also bounded. Thus $v_1 \in L^2(\Ycal)$, and the same argument applies inductively to all subsequent layers.

We also note that, in some implementations of neural operators, a single layer is defined as
\[
\operatorname{N}(v)(y) = \sigma \left( A\,v(y) + \int_{\mathcal{X}} k(y, x)\, v(x)\, dx + b(y) \right),
\]
where $ A \in \mathbb{R} $ is a scalar parameters for scalar-valued functions. To derive an analog of Proposition~\ref{prop:holder-NO}, we must additionally assume that the input function $ v $ lies in $ C^{0,\alpha}(\Ycal) $, rather than merely being square-integrable. Under this assumption, the multilayer extension follows immediately, as each layer maps functions in $ C^{0,\alpha}(\Ycal) $ to the same space, preserving H{\" o}lder continuity throughout the network.

\medskip 

\noindent \textbf{Fourier Neural Operators.}
Finally, we conclude this section by considering the concrete example of the Fourier Neural Operator (FNO) \citep{li2020fourier}. In the FNO architecture, the kernel is defined via a truncated Fourier series. That is,
\[k_{\text{FNO}}(y,x) = \sum_{|m|_{\infty}\leq K} \lambda_m \, e^{2\pi \operatorname{i} m \cdot (y-x) },\]
where $\lambda_m$'s are the learned parameters. For any $y_1,y_2$, we have
\begin{equation*}
    \begin{split}
     |k_{\text{FNO}}(y_1,x)-k_{\text{FNO}}(y_2,x)| 
     &=   \left|\sum_{|m|_{\infty}\leq K} \lambda_m \, e^{2\pi \operatorname{i} m \cdot (y_1-x) }- \sum_{|m|_{\infty}\leq K} \lambda_m \, e^{2\pi \operatorname{i} m \cdot (y_2-x) } \right|\\
     &\leq \sum_{|m|_{\infty}\leq K} |\lambda_m|\cdot |e^{2\pi \operatorname{i}m\cdot x}| \cdot |e^{2\pi \operatorname{i} m \cdot y_1 }-e^{2\pi \operatorname{i} m \cdot y_2} |\\
    \end{split}
\end{equation*}
Note that $|e^{2\pi \operatorname{i}m\cdot x}| \leq 1$. Since the gradient of the function $y \mapsto e^{2\pi \operatorname{i}m\cdot y}$ is $ 2\pi \operatorname{i}\, m \, e^{2\pi \operatorname{i}m\cdot y}$, the mean value theorem implies that this function is Lipschitz with constant  $|2\pi \operatorname{i}\, m|_2 \leq 2\pi |m|_2$. Therefore, 
\[|k_{\text{FNO}}(y_1,x)-k_{\text{FNO}}(y_2,x)| \leq 2\pi \left( \sum_{|m|_{\infty}\leq K} |\lambda_m| |m|_2 \right) \, |y_1-y_2|_2. \]
In other words, $k_{\text{FNO}}$ is uniformly Lipschitz in $y$ with Lipschitz constant $2\pi \left( \sum_{|m|_{\infty}\leq K} |\lambda_m| |m|_2 \right)$. Therefore, Proposition \ref{prop:holder-NO} implies that the output of FNO under appropriate assumptions are H{\" o}lder continuous.

\section{When Inputs Are Available Only on a Discrete Grid of $\Xcal$}
\label{sec:discreteX}

Thus far, we have assumed that each input function is available on the full continuum during both
training and inference. In this section, we briefly discuss how our results extend to the more
realistic setting in which input functions are observed only on discrete grids.

Suppose that the input functions $v$ are accessible only on a discrete input grid
$\Xcal_{\mathrm{train}}$ during training, and on a possibly finer grid
$\Xcal_{\mathrm{test}} \supseteq \Xcal_{\mathrm{train}}$ at test time. We define the expected errors
when the estimator is evaluated on the training and testing grids as
\[
\Ecal_{\mu}\!\left(
\widehat{\operatorname{F}}_n, \operatorname{G},
\Xcal_{\mathrm{train}}, \Ycal_{\mathrm{train}}
\right)
:=
\expect_{v \sim \mu}
\left[
\frac{1}{|\Ycal_{\mathrm{train}}|}
\sum_{y \in \Ycal_{\mathrm{train}}}
\left(
\widehat{\operatorname{F}}_n\!\left(v_{\mid \Xcal_{\mathrm{train}}}\right)(y)
-
\operatorname{G}(v)(y)
\right)^2
\right],
\]
and
\[
\Ecal_{\mu}\!\left(
\widehat{\operatorname{F}}_n, \operatorname{G},
\Xcal_{\mathrm{test}}, \Ycal_{\mathrm{test}}
\right)
:=
\expect_{v \sim \mu}
\left[
\frac{1}{|\Ycal_{\mathrm{test}}|}
\sum_{y \in \Ycal_{\mathrm{test}}}
\left(
\widehat{\operatorname{F}}_n\!\left(v_{\mid \Xcal_{\mathrm{test}}}\right)(y)
-
\operatorname{G}(v)(y)
\right)^2
\right].
\]

\noindent Note that
$\Ecal_{\mu}(\widehat{\operatorname{F}}_n, \operatorname{G}, \Xcal_{\mathrm{train}}, \Ycal_{\mathrm{train}})$
is still a \emph{test-time} error, not the empirical training error on the training dataset. The
subscript ``$\mathrm{train}$'' indicates only that the estimator is evaluated on the same input and
output grids on which the training data were provided.

For zero-shot super-resolution in settings where input functions are accessible only on a discrete
grid $\Xcal$, we require that the estimator $\widehat{\operatorname{F}}_n$ satisfies a suitable form
of discretization invariance, following the notion introduced in
\citep{kovachki2021universal}. Specifically, assume there exists a nested sequence of input grids
$\{\Xcal_k\}_{k \in \mathbb{N}}$ such that $\Xcal_k \subseteq \Xcal_{k+1}$ and
$\lim_{k \to \infty} \Xcal_k = \Xcal$, for which the estimator satisfies
\[
\lim_{k \to \infty}
\expect_{v \sim \mu}
\left[
\frac{1}{|\Ycal_{\mathrm{test}}|}
\sum_{y \in \Ycal_{\mathrm{test}}}
\left(
\widehat{\operatorname{F}}_n(v)(y)
-
\widehat{\operatorname{F}}_n\!\left(v_{\mid \Xcal_k}\right)(y)
\right)^2
\right]
=
0.
\]
The notion of discretization invariance used here is defined with respect to the discrete output
grid $\Ycal_{\mathrm{test}}$, rather than the continuous $L^2(\Ycal)$ norm considered in
\citep{kovachki2023neural}. However, an inspection of the proof of Theorem~8 in
\citep{kovachki2023neural} shows that their argument first establishes uniform convergence at every
point $y \in \Ycal$ (see Equation~(50) therein), and then uses the boundedness of $\Ycal$ to extend
this pointwise result to convergence in the continuum norm. Since our setting requires convergence
only on a finite output grid, this weaker form of discretization invariance is sufficient. Finally, while the result of \cite{kovachki2023neural} holds uniformly over all inputs $v$ in a
compact subset $\Acal \subseteq \Vcal$, we formulate our condition in expectation with respect to a
distribution $\mu$. If $\mu$ is supported on a compact subset of $\Vcal$, their uniform guarantee
directly implies the condition above.

We now show that, in addition to the conditions in Section~\ref{sec:upper}, discretization
invariance is sufficient to guarantee zero-shot super-resolution when inputs are observed only on
discrete grids. Since discretization invariance is a qualitative property, it does not by itself
yield a convergence rate. To make this notion quantitative, we assume the existence of a sequence
$\{\varepsilon_k\}_{k \in \mathbb{N}}$ with $\varepsilon_k \to 0$ as $k \to \infty$ such that
\begin{equation}
\label{eq:disc-inv}
\expect_{v \sim \mu}
\left[
\frac{1}{|\Ycal_{\mathrm{test}}|}
\sum_{y \in \Ycal_{\mathrm{test}}}
\left(
\widehat{\operatorname{F}}_n(v)(y)
-
\widehat{\operatorname{F}}_n\!\left(v_{\mid \Xcal_k}\right)(y)
\right)^2
\right]
\le
\varepsilon_k.
\end{equation}
This condition provides a rate-controlled version of discretization invariance with respect to the
discrete output grid $\Ycal_{\mathrm{test}}$.

\medskip

We then obtain the following bound.
\begin{theorem}
\label{thm:ub-discreteinp}
Assume that $\Xcal_{\mathrm{train}} = \Xcal_{k_1}$ and $\Xcal_{\mathrm{test}} = \Xcal_{k_2}$ for some
$k_1, k_2 \in \naturals$. For any estimator $\widehat{\operatorname{F}}_n$ satisfying
\eqref{eq:disc-inv}, we have
\[
\begin{aligned}
\Ecal_{\mu} \left(
\widehat{\operatorname{F}}_n, \operatorname{G},
\Xcal_{\mathrm{test}}, \Ycal_{\mathrm{test}}
\right)
&\le
2(\varepsilon_{k_1} + \varepsilon_{k_2})
+
2 \frac{|\Ycal_{\mathrm{train}}|}{|\Ycal_{\mathrm{test}}|}
(2\nu - 1)\,
\Ecal_{\mu}\!\left(
\widehat{\operatorname{F}}_n, \operatorname{G},
\Xcal_{\mathrm{train}}, \Ycal_{\mathrm{train}}
\right)
\\
&\quad
+
16\left(
1 - \frac{|\Ycal_{\mathrm{train}}|}{|\Ycal_{\mathrm{test}}|}
\right)
c^2 \beta^{2\alpha}.
\end{aligned}
\]
\end{theorem}

\noindent
In summary, Theorem \ref{thm:ub-discreteinp} shows that, beyond the conditions in
Section \ref{sec:upper}, discretization invariance of the estimator with respect to the discrete
output norm is sufficient to ensure zero-shot super-resolution when inputs are observed only on
discrete grids. The additional error terms $\varepsilon_{k_1}$ and $\varepsilon_{k_2}$ quantify the
loss incurred by approximating continuous inputs using discrete samples and vanish as the input
grids are refined. In particular, when these discretization errors are small and the coverage and
load-balancing conditions are satisfied, reliable cross-resolution generalization is achievable.

\section{Experiments}

In this section, we present empirical results illustrating the failure modes predicted by the
impossibility result in Theorem~\ref{thm:lb-general}. We demonstrate these failures in two simple
settings. First, we consider the synthetic ground-truth operator used in the
construction of our lower bound. Second, we study the inviscid Burgers equation, which is known to
produce irregular solutions and therefore poses challenges for cross-resolution generalization. As our model of choice, we use Fourier Neural Operators (FNOs), which have been shown to be among
the most effective architectures for zero-shot super-resolution across a wide range of empirical
tasks. Since the success of neural operators in zero-shot super-resolution has already been
demonstrated extensively in prior works
\citep{li2020fourier, jiang2023efficient, luo2024hierarchical, yang2024fourier, yasuda2025zero,
sinha2025effectiveness}, we focus exclusively on failure
cases. Our experiments are intended primarily for illustration of our theoretical findings rather than exhaustive evaluation.
For a more comprehensive empirical studies documenting the limitations of zero-shot super-resolution,
we refer the reader to \citep{sakarvadia2025false, gao2025discretization}.

\subsection{Synthetic Data: Lower-Bound Setup}
The input functions $v$ are Gaussian random fields sampled on a
uniform grid of $[0,1)$. Outputs are generated using the operator $\operatorname{G}_\omega$, defined by
\[
 (\operatorname{G}_\omega v)(y) = f_\omega(y) \int_0^1 x\, v(x)\, dx,
\]
where $f_\omega(y) \in \{-1, +1\}$ is a Rademacher function defined on the output grid. This
construction isolates resolution dependence: while the scalar inner product can be accurately
learned from coarse data, correct prediction on a finer output grid requires resolving the
unobserved sign pattern $f_\omega$. We train models at resolution $N=128$ and evaluate them at test resolutions $N=128$, $256$, and
$512$. As shown in Table~1, performance is excellent at the training resolution but deteriorates
dramatically on finer grids.  Figure \ref{fig:residual_plot} further illustrates this failure through residual plots comparing the predicted and ground-truth output functions.
\begin{table}[ht]
\centering
\begin{tabular}{|c|c|}
\hline
\textbf{Test Resolution $(N)$} & \textbf{Relative $L^2$ Error} \\
\hline
\hline
128 & 0.0066 \\
256 &0.9614 \\
512 & 1.2331 \\
\hline
\end{tabular}
\caption{Test performance of a model trained at resolution $N=128$ using $2056$ samples. 
Errors are reported as relative $L^2$ norms, averaged over $640$ random test inputs.}
\end{table}

\subsection{Inviscid Burgers Equation}

We next evaluate zero-shot generalization on the one-dimensional inviscid Burgers equation,
\[
\partial_t u(x,t) + \partial_x\!\Big(\tfrac{1}{2}u(x,t)^2\Big) = 0, 
\qquad x \in [0,1], \; t \in [0,T],
\]
with periodic boundary conditions. Initial conditions are generated as random superpositions of $m$ Fourier modes, $
u_0(x) = \sum_{j=1}^{m} 
A_j \sin(2\pi k x + \phi_j) + 
B_j \cos(2\pi k x + \psi_j),
$
where $A_j,B_j \sim \mathrm{Unif}[1.0,5.0]$, fixed wavenumber $k=8$, and phases $\phi_j,\psi_j \sim \mathrm{Unif}[0,2\pi)$. This construction yields smooth random fields with moderate oscillations. Each initial condition is evolved to time $T=0.2$ using a Godunov finite-volume scheme on a fine grid ($N_{\mathrm{hi}}=512$). The resulting data are anti-aliased and downsampled to coarser grids for training and evaluation.

A Fourier Neural Operator is trained on $N=128$-point grids using $640$ samples and evaluated at test resolutions $N=128,256,$ and $512$. Table~\ref{tab:burgers_sr} reports relative $L^2$ errors averaged over $128$ test samples. The model attains high accuracy on the training grid but fails to generalize to finer resolutions.

\begin{table}[H]
\centering
\begin{tabular}{|c|c|}
\hline
\textbf{Test Resolution $(N)$} & \textbf{Relative $L^2$ Error} \\
\hline
\hline
128 & 0.0578 \\
256 & 0.1606 \\
512 & 0.2047 \\
\hline
\end{tabular}
\caption{Test performance of a model trained at resolution $N=128$ on smooth Burgers data. 
Errors are reported as relative $L^2$ norms averaged over $128$ random test inputs.}
\label{tab:burgers_sr}
\end{table}

\noindent Figure~\ref{fig:residual_plot_burgers} compares the ground truth $u$, predictions $\hat{u}$, and residuals $r=\hat{u}-u$ across test resolutions. On the training grid, predictions coincide almost exactly with the reference solution. As the grid is refined, however, residuals increase and develop structured oscillations, particularly near the shock where $u$ transitions sharply between positive and negative values. This degradation in cross-resolution generalization in nonsmooth regions highlights the necessity of H{\"o}lder regularity or related smoothness assumptions
for zero-shot super-resolution.

\begin{figure}[H]
    \centering
    \includegraphics[width=\linewidth]{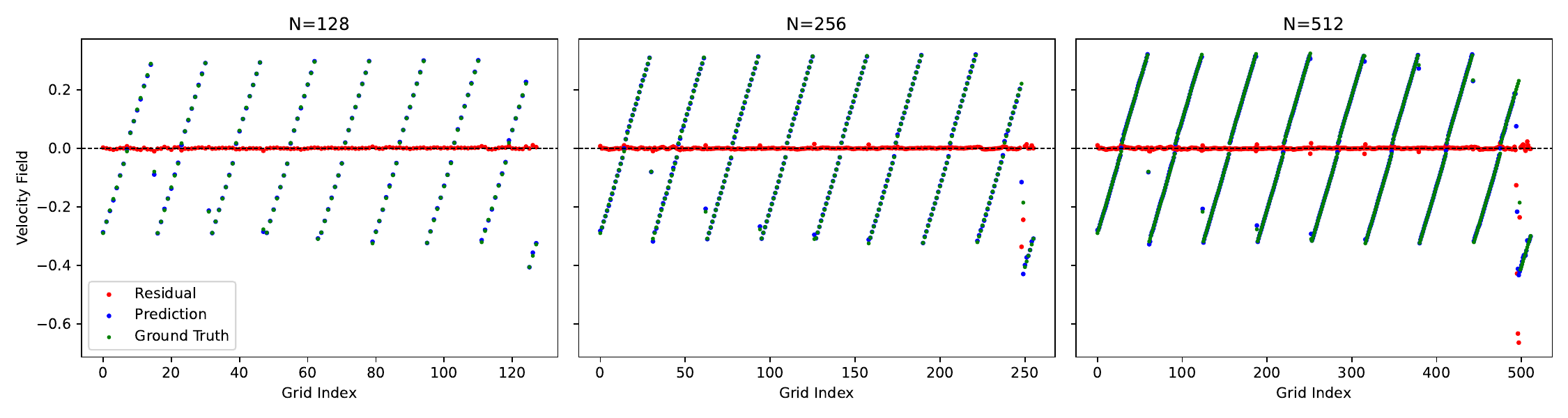}
    \caption{Ground truth $u$, predictions $\hat{u}$, and residuals $r=\hat{u}-u$ 
    across test resolutions $N=128,256,512$. 
    }
    \label{fig:residual_plot_burgers}
\end{figure}

\section{Discussion and Future Work}
Our work initiates a theoretical investigation of zero-shot super-resolution in operator learning, establishing both an impossibility result and sufficient conditions for its feasibility. Our impossibility result shows that zero-shot super-resolution is not guaranteed in general even for simple operators and highly expressive models. This suggest that zero-shot super-resolution should be understood as a non-trivial extrapolation problem rather than a consequence of discretization invariance or model capacity. Our positive results identify H{\" o}lder regularity as a key structural assumption under which zero-shot super-resolution becomes achievable. Finally, we also provide a generalization bound that explains the empirical success observed in smooth problem settings.

Two major directions remain open for future work. One natural extension is to establish tighter bounds for function classes characterized by different notions of regularity, such as Sobolev or Besov spaces. Second, understanding the trade-offs between strict zero-shot guarantees and multi-resolution training with limited higher-resolution samples, as suggested by \cite{sakarvadia2025false}, is also an important direction for developing reliable cross-resolution operator-learning methods.

\section*{Acknowledgements}
We acknowledge the support of NSF via grant DMS-2413089.
US also acknowledges the support of the Rackham Predoctoral Fellowship.

\newpage
\bibliographystyle{plainnat}
\bibliography{arxiv/references}
\newpage 

\appendix

\section{Proof of Theorem \ref{thm:lb-general}}\label{appdx:proof_lb}

\subsection{Constructing the Class $\Gcal$}
Define a function $g: [0,1)^2 \to [0,1]$ such that
\[g(y,x)=yx.\]
Recall that the integral operator of the function $g(y,x)$ is an operator $\operatorname{G}_{\text{base}}$ such that
\[\big(\operatorname{G}_{\text{base}}(v)\big)(y)=\int_{0}^1 g(y,x)\, v(x)\, dx.\]
We will first modify $g$ for certain values of $y$ and take our ground truth to be the integral operator of that modified $g$.

For any sequence $ \xi := \{\xi_y\}_{y \in \mathbb{Q} \cap [0,1)} $, where each $ \xi_y \sim \mathrm{Unif}(\{-1,1\}) $ independently, define a function $ f_{\xi} : [0,1) \to [-1,1] $ by
\[
f_{\xi}(y) := 
\begin{cases}
\xi_y, & \text{if } y \in \mathbb{Q} \cap [0,1), \\[6pt]
y, & \text{otherwise}.
\end{cases}
\]
Since $ f_{\xi}(y) = y $ almost everywhere and the identity function $ y \mapsto y $ is measurable, it follows by \citet[Proposition 2.11]{folland1999real} that $ f_{\xi} $ is also measurable. 

Now, for each sequence $ \xi \in \{-1,1\}^{\mathbb{Q} \cap [0,1)} $, define a function $ g_{\xi} : [0,1) \times [0,1) \to \mathbb{R} $ by
\[
g_{\xi}(y, x) := f_{\xi}(y)\, x.
\]
Note that $g(y,x) = g_{\xi}(y,x)$ almost everywhere. 

Let $ \operatorname{G}_{\xi} $ denote the integral operator associated with this function.  By construction, we have
\[
\operatorname{G}_{\xi}(v)(y) = f_{\xi}(y) \cdot \int_0^1 x \,v(x)\, dx,
\]
which implies that $ \operatorname{G}_{\xi}(v) \in \text{span}(f_{\xi}) $ for all $ v \in L^2([0,1)) $. That is, $ \operatorname{G}_{\xi} $ is a rank-one operator. Moreover,
\begin{equation*}
    \begin{split}
        \norm{\operatorname{G}_{\xi}}_{\text{op}}^2 = \norm{\operatorname{G}_{\xi}}_{\text{HS}}^2 = \int_{0}^1 \int_{0}^1 |g_{\xi}(y,x)|^2\, dy\, dx = \int_{0}^1 \int_{0}^1 |g(y,x)|^2\, dy\, dx
        &= \left( \int_{0}^1 r^2\, dr\right)^2\\
        &= \frac{1}{9}.
    \end{split}
\end{equation*}
Thus, $     \norm{\operatorname{G}_{\xi}}_{\text{op}} \leq 1$. Define a class of operators
\[\Gcal := \left\{\operatorname{G}_{\xi}\, \mid \xi \in \{-1,1\}^{\rationals \cap [0,1)} \right\}.\]

\subsection{Proof of part (i)}

Let $ \mu $ be any probability measure supported on $ L^2(\mathcal{X}) $, such that every $ v \sim \mu $ satisfies $ 0 < a \leq |v(x)| \leq 1 $ for all $x \in \Xcal$. Suppose the ground truth operator $ \operatorname{G} \in \Gcal $. By definition of $ \Gcal $, there exists a sequence $ \xi := \{\xi_y\}_{y \in \mathbb{Q} \cap [0,1)} $ such that $ \operatorname{G} = \operatorname{G}_{\xi} $.

Consider any sample $ (v_1, w_1) $. Then,
\[
w_1(y) = \operatorname{G}_{\xi}(v_1)(y) = \int_{0}^{1} g_{\xi}(y,x)\, v_1(x)\, dx = f_{\xi}(y) \int_{0}^{1} x \, v_1(x)\, dx.
\]
Since $ f_{\xi}(y) = \xi_y $ for all $ y \in \Ycal_{\mathrm{train}} \subseteq \mathbb{Q} \cap [0,1) $, and
\[
\int_0^1 x v_1(x)\, dx \geq c \int_0^1 x\, dx = \frac{c}{2} > 0,
\]
we can recover $ \xi_y $ for all $ y \in \Ycal_{\mathrm{train}} $ via
\[
\xi_y = \frac{w_1(y)}{\int_0^1 x v_1(x)\, dx}.
\]
Now define an estimator $ \widehat{\operatorname{F}}^{\prime} := \operatorname{G}_{\xi'} $ such that
\[
\xi_y^{\prime} = \xi_y \quad \text{for all } y \in \Ycal_{\mathrm{train}}.
\]
Then, for any $ v \in L^2(\Xcal) $,
\[
\widehat{\operatorname{F}}^{\prime}(v)(y) = \xi_y^{\prime} \int_0^1 x v(x)\, dx = \xi_y \int_0^1 x v(x)\, dx = \operatorname{G}_{\xi}(v)(y), \quad \forall y \in \Ycal_{\mathrm{train}}.
\]
Therefore,
\[
\mathbb{E}_{v \sim \mu} \left[ \frac{1}{|\Ycal_{\mathrm{train}}|} \sum_{y \in \Ycal_{\mathrm{train}}} \left( \widehat{\operatorname{F}}^{\prime}(v)(y) - \operatorname{G}(v)(y) \right)^2 \right] = 0.
\]
This completes our proof of part (i).

\subsection{Proof of part (ii)}
Let $ \mu $ be any probability measure supported on $ L^2(\mathcal{X}) $ such that every $ v \sim \mu $ satisfies $ 0 < a \leq |v(x)| \leq 1$ for all $x \in \Xcal$.

Our proof is based on probabilistic method. In particular, we  show that for any estimation rule $ \widehat{\operatorname{F}}_n $, the following bound holds
\begin{equation}\label{eq:lb}
    \expect_{\xi} \left[ \expect_{v_1, \ldots, v_n \sim \mu} \left[ \expect_{v \sim \mu} \left[ 
    \frac{1}{|\Ycal_{\text{test}}|} \sum_{y \in \Ycal_{\text{test}}} \left( \widehat{\operatorname{F}}(v)(y) - \operatorname{G}_{\xi}(v)(y) \right)^2
    \right] \right] \right] 
    \geq \frac{a^2}{4} \left( 1 - \frac{N_1}{N_2} \right).
\end{equation}
By a standard probabilistic argument (see \citet{alon2016probabilistic}), it follows that there exists a particular sequence $ \xi^\star $ and the corresponding operator $ \operatorname{G} := \operatorname{G}_{\xi^\star} $ such that
\[
\expect_{v_1, \ldots, v_n \sim \mu} \left[ \expect_{v \sim \mu} \left[ 
   \frac{1}{|\Ycal_{\text{test}}|} \sum_{y \in \Ycal_{\text{test}}} \left( \widehat{\operatorname{F}}(v)(y) - \operatorname{G}^{\star}(v)(y) \right)^2 
\right] \right] 
\geq \frac{a^2}{4} \left( 1 - \frac{N_1}{N_2} \right).
\]
Using the fact that $N_2=mN_1$ completes our proof.

We now proceed to prove Equation~\eqref{eq:lb}. We can rewrite the left-hand side of Equation~\eqref{eq:lb} and restrict the summation to points in the test grid not included in the training grid as
\begin{equation*}
\begin{split}
&\expect_{\xi} \left[ \expect_{v_1, \ldots, v_n \sim \mu} \left[ \expect_{v \sim \mu} \left[ 
    \frac{1}{|\Ycal_{\text{test}}|} \sum_{y \in \Ycal_{\text{test}}} \left( \widehat{\operatorname{F}}(v)(y) - \operatorname{G}_{\xi}(v)(y) \right)^2
    \right] \right] \right] \\
&\geq \expect_{v_1, \ldots, v_n \sim \mu} \left[ \expect_{v \sim \mu} \left[ \expect_{\xi} \left[
 \frac{1}{|\Ycal_{\text{test}}|} \sum_{y \in \mathcal{Y}_{\text{test}} \setminus \mathcal{Y}_{\text{train}}} 
\left( \widehat{\operatorname{F}}_n(v)(y) - \operatorname{G}_{\xi}(v)(y) \right)^2 
\right] \right] \right],
\end{split}
\end{equation*}
where the final step follows by Fubini's theorem, since the random variables $ v_1, \ldots, v_n \sim \mu $, the test sample $ v \sim \mu $, and the sequence $ \xi \sim \mathrm{Unif}(\{-1,1\})^{\mathbb{Q} \cap [0,1)} $ are all drawn independently. In particular, the order of integration over $ \xi $, the training data, and the test data can be exchanged freely.

Note that for any $ y \in \mathcal{Y}_{\text{test}} \setminus \mathcal{Y}_{\text{train}} $, the squared error can be lower bounded as 
\begin{equation*}
\begin{split}
\left( \widehat{\operatorname{F}}_n(v)(y) - \operatorname{G}_{\xi}(v)(y) \right)^2 
&= \left( \widehat{\operatorname{F}}_n(v)(y) - \xi_y \int_0^1 x v(x) \, dx \right)^2 \\
&= \left( \widehat{\operatorname{F}}_n(v)(y) \right)^2 
- 2 \widehat{\operatorname{F}}_n(v)(y) \cdot \xi_y \int_0^1 x v(x) \, dx 
+ \left( \int_0^1 x v(x) \, dx \right)^2 \\
&\geq -2 \widehat{\operatorname{F}}_n(v)(y) \cdot \xi_y \int_0^1 x v(x) \, dx 
+ \left( \int_0^1 x v(x) \, dx \right)^2.
\end{split}
\end{equation*}

Next, we show that the first term vanishes in expectation. More precisely, for fixed training samples $ v_1, \ldots, v_n $ and a test point $ v \sim \mu $, we have
\begin{equation*}
\begin{split}
\expect_{\xi} \left[ \widehat{\operatorname{F}}_n(v)(y) \cdot \xi_y \int_0^1 x v(x)\, dx \right]
&= \expect_{\xi \setminus \xi_y} \left[ \expect_{\xi_y} \left[ \widehat{\operatorname{F}}_n(v)(y) \cdot \xi_y \int_0^1 x v(x)\, dx \,\bigg|\, \xi \setminus \xi_y \right] \right] \\
&= \expect_{\xi \setminus \xi_y} \left[ \widehat{\operatorname{F}}_n(v)(y) \cdot \left( \int_0^1 x v(x)\, dx \right) \cdot \expect_{\xi_y} \left[ \xi_y \,\big|\, \xi \setminus \xi_y \right] \right] \\
&= 0,
\end{split}
\end{equation*}
since $ \expect_{\xi_y} [\xi_y \mid \xi \setminus \xi_y] = 0 $. The first step uses the law of iterated expectation. The second step follows from the fact that, for any fixed $ v_1, \ldots, v_n $ and $ v $, the random variable $ \xi_y $ is conditionally independent of $ \widehat{\operatorname{F}}_n(v)(y) $ given $ \xi \setminus \xi_y $. This is because the estimator $ \widehat{\operatorname{F}}_n $ only has access to the values of $ \xi_y $ at the training grid $ \mathcal{Y}_{\text{train}} $. For any $ y \notin \mathcal{Y}_{\text{train}} $, the learner receives no information about $ \xi_y $, and therefore the estimator’s prediction at such $ y $ must be independent of $ \xi_y $ conditioned on $ \xi \setminus \xi_y $.

To bound the second term, observe that since $ v(x) \geq a $ for all $ x \in [0,1) $, we have
\[
\left( \int_0^1 x v(x)\, dx \right)^2 
\geq \left( \int_0^1 x \cdot a \, dx \right)^2 
= a^2 \left( \int_0^1 x\, dx \right)^2 
= \frac{a^2}{4}.
\]
Combining this bound with the previous steps, we conclude that
\begin{equation*}
\begin{split}
&    \expect_{\xi} \left[ \expect_{v_1, \ldots, v_n \sim \mu} \left[ \expect_{v \sim \mu} \left[ 
    \frac{1}{|\Ycal_{\text{test}}|} \sum_{y \in \Ycal_{\text{test}}} \left( \widehat{\operatorname{F}}(v)(y) - \operatorname{G}_{\xi}(v)(y) \right)^2
    \right] \right] \right] \\
&\geq \frac{1}{\mathcal{Y}_{\text{test}}} \sum_{y \in \mathcal{Y}_{\text{test}} \setminus \mathcal{Y}_{\text{train}}} \frac{a^2}{4} \\
&= \frac{a^2}{4} \cdot \frac{|\mathcal{Y}_{\text{test}} \setminus \mathcal{Y}_{\text{train}}|}{|\mathcal{Y}_{\text{test}}|}.
\end{split}
\end{equation*}
Noting that $|\mathcal{Y}_{\text{test}} \setminus \mathcal{Y}_{\text{train}}| \geq |\mathcal{Y}_{\text{test}} |- |\mathcal{Y}_{\text{train}} |$ completes our proof of part (ii).

\section{Proof of Theorem \ref{thm:ub}}\label{appdx:ub}
\begin{proof} 
We begin by decomposing the error over the test grid as
\begin{equation*}
\begin{split}
&\mathbb{E}_{v \sim \mu} \left[ \frac{1}{|\mathcal{Y}_{\mathrm{test}}|} \sum_{y \in \mathcal{Y}_{\mathrm{test}}} \left( \widehat{\operatorname{F}}_n(v)(y) - \operatorname{G}(v)(y) \right)^2 \right]\\
&= \mathbb{E}_{v \sim \mu} \left[ \frac{1}{|\mathcal{Y}_{\mathrm{test}}|} \sum_{y \in \mathcal{Y}_{\mathrm{train}}} \left( \widehat{\operatorname{F}}_n(v)(y) - \operatorname{G}(v)(y) \right)^2 \right] \quad \\
&+ \mathbb{E}_{v \sim \mu} \left[ \frac{1}{|\mathcal{Y}_{\mathrm{test}}|} \sum_{y \in \mathcal{Y}_{\mathrm{test}} \setminus \mathcal{Y}_{\mathrm{train}}} \left( \widehat{\operatorname{F}}_n(v)(y) - \operatorname{G}(v)(y) \right)^2 \right].
\end{split}
\end{equation*}

For each $ y \in \mathcal{Y}_{\mathrm{test}} \setminus \mathcal{Y}_{\mathrm{train}} $, let $ \mathrm{nn}(y) \in \mathcal{Y}_{\mathrm{train}} $ denote a nearest neighbor of $ y $ in the training grid, that is
\[
|y - \mathrm{nn}(y)|_2 \leq |y - y'|_2 \quad \text{for all } y' \in \mathcal{Y}_{\mathrm{train}}.
\]
Here, ties may be broken arbitrarily. Then for any such $ y $, we have
\begin{equation*}
\begin{split}
&\left| \widehat{\operatorname{F}}_n(v)(y) - \operatorname{G}(v)(y) \right| \\
&= \left| \widehat{\operatorname{F}}_n(v)(y) - \widehat{\operatorname{F}}_n(v)(\mathrm{nn}(y)) + \widehat{\operatorname{F}}_n(v)(\mathrm{nn}(y)) - \operatorname{G}(v)(\mathrm{nn}(y)) + \operatorname{G}(v)(\mathrm{nn}(y)) - \operatorname{G}(v)(y) \right| \\
&\leq \left| \widehat{\operatorname{F}}_n(v)(y) - \widehat{\operatorname{F}}_n(v)(\mathrm{nn}(y)) \right| 
+ \left| \widehat{\operatorname{F}}_n(v)(\mathrm{nn}(y)) - \operatorname{G}(v)(\mathrm{nn}(y)) \right| 
+ \left| \operatorname{G}(v)(\mathrm{nn}(y)) - \operatorname{G}(v)(y) \right| \\
&\leq c |y - \mathrm{nn}(y)|_2^\alpha 
+ \left| \widehat{\operatorname{F}}_n(v)(\mathrm{nn}(y)) - \operatorname{G}(v)(\mathrm{nn}(y)) \right| 
+ c |y - \mathrm{nn}(y)|_2^\alpha \\
&= 2c |y - \mathrm{nn}(y)|_2^\alpha 
+ \left| \widehat{\operatorname{F}}_n(v)(\mathrm{nn}(y)) - \operatorname{G}(v)(\mathrm{nn}(y)) \right|,
\end{split}
\end{equation*}
where we have used the assumption that both $ \widehat{\operatorname{F}}_n(v) $ and $ \operatorname{G}(v) $ are uniformly Hölder continuous with exponent $ \alpha $ and constant $ c $. Using the inequality $(a+b)^2 \leq 2(a^2+b^2)$, we have
\begin{equation*}
    \begin{split}
      \left| \widehat{\operatorname{F}}_n(v)(y) - \operatorname{G}(v)(y) \right|^2 &\leq 8c^2 |y - \mathrm{nn}(y)|_2^{2\alpha}
+ 2\left| \widehat{\operatorname{F}}_n(v)(\mathrm{nn}(y)) - \operatorname{G}(v)(\mathrm{nn}(y)) \right|^2\\
&\leq  8c^2 \beta^{2\alpha}
+ 2\left| \widehat{\operatorname{F}}_n(v)(\mathrm{nn}(y)) - \operatorname{G}(v)(\mathrm{nn}(y)) \right|^2. 
    \end{split}
\end{equation*}

Thus, we have
\begin{equation*}
\begin{split}
&\expect_{v \sim \mu} \left[ \frac{1}{|\mathcal{Y}_{\mathrm{test}}|} \sum_{y \in \mathcal{Y}_{\mathrm{test}} \setminus \mathcal{Y}_{\mathrm{train}}} \left( \widehat{\operatorname{F}}_n(v)(y) - \operatorname{G}(v)(y) \right)^2 \right] \\
&\leq  \frac{|\Ycal_{\mathrm{test}} \setminus \Ycal_{\mathrm{train}}|}{|\Ycal_{\mathrm{test}}|}\cdot 8c^2 \beta^{2\alpha}+ 2\expect_{v \sim \mu} \left[ \frac{1}{|\mathcal{Y}_{\mathrm{test}}|} \sum_{y \in \mathcal{Y}_{\mathrm{test}} \setminus \mathcal{Y}_{\mathrm{train}}} \left| \widehat{\operatorname{F}}_n(v)(\mathrm{nn}(y)) - \operatorname{G}(v)(\mathrm{nn}(y)) \right|^2. \right]
\end{split}
\end{equation*}

Note that
\begin{equation*}
    \begin{split}
        \sum_{y \in \mathcal{Y}_{\mathrm{test}} \setminus \mathcal{Y}_{\mathrm{train}}} &\left| \widehat{\operatorname{F}}_n(v)(\mathrm{nn}(y)) - \operatorname{G}(v)(\mathrm{nn}(y)) \right|^2\\
        &= \sum_{z \in \Ycal_{\mathrm{train}}}   \sum_{y \in \mathcal{Y}_{\mathrm{test}} \setminus \mathcal{Y}_{\mathrm{train}}} \left| \widehat{\operatorname{F}}_n(v)(z) - \operatorname{G}(v)(z) \right|^2 \indicator[\text{nn}(y)=z]\\
        &\leq \left(\sup_{z \in \Ycal_{\mathrm{train}}}  \sum_{y \in \mathcal{Y}_{\mathrm{test}} \setminus \mathcal{Y}_{\mathrm{train}}} \indicator[\mathrm{nn}(y)=z]\right)\, \cdot  \sum_{z \in \Ycal_{\mathrm{train}}} \left| \widehat{\operatorname{F}}_n(v)(z) - \operatorname{G}(v)(z) \right|^2\\
        &= (\nu-1) \cdot  \sum_{z \in \Ycal_{\mathrm{train}}} \left| \widehat{\operatorname{F}}_n(v)(z) - \operatorname{G}(v)(z) \right|^2.
    \end{split}
\end{equation*}

Note that the final equality holds because the point itself is the nearest neighbor for all $y \in \Ycal_{\mathrm{train}}$.  Thus, we have shown that 
\begin{equation*}
\begin{split}
&\expect_{v \sim \mu} \left[ \frac{1}{|\mathcal{Y}_{\mathrm{test}}|} \sum_{y \in \mathcal{Y}_{\mathrm{test}} \setminus \mathcal{Y}_{\mathrm{train}}} \left( \widehat{\operatorname{F}}_n(v)(y) - \operatorname{G}(v)(y) \right)^2 \right] \\
&\leq  \frac{|\Ycal_{\mathrm{test}} \setminus \Ycal_{\mathrm{train}}|}{|\Ycal_{\mathrm{test}}|}\cdot 8c^2 \beta^{2\alpha}  + 2(\nu-1) \cdot \expect_{v \sim \mu} \left[ \frac{1}{|\mathcal{Y}_{\mathrm{test}}|} \sum_{z \in \Ycal_{\mathrm{train}}} \left| \widehat{\operatorname{F}}_n(v)(z) - \operatorname{G}(v)(z) \right|^2 \right]\\
&= \frac{|\Ycal_{\mathrm{test}} \setminus \Ycal_{\mathrm{train}}|}{|\Ycal_{\mathrm{test}}|}\cdot 8c^2 \beta^{2\alpha}  + \frac{|\mathcal{Y}_{\mathrm{train}}| }{|\Ycal_{\mathrm{test}}|} \cdot 2(\nu-1) \cdot \mathcal{E}_{\mu}(\widehat{\operatorname{F}}_n, \operatorname{G}, \Ycal_{\mathrm{train}}).
\end{split}
\end{equation*}
Therefore, by combining everything, we have shown that
\begin{equation*}
    \begin{split}
       \mathbb{E}_{v \sim \mu} &\left[ \frac{1}{|\mathcal{Y}_{\mathrm{test}}|} \sum_{y \in \mathcal{Y}_{\mathrm{test}}} \left( \widehat{\operatorname{F}}_n(v)(y) - \operatorname{G}(v)(y) \right)^2 \right] \\
       &\leq \frac{|\mathcal{Y}_{\mathrm{train}}| }{|\Ycal_{\mathrm{test}}|} \cdot (2\nu-1) \cdot \mathcal{E}_{\mu}(\widehat{\operatorname{F}}_n, \operatorname{G}, \Ycal_{\mathrm{train}})  + \frac{|\Ycal_{\mathrm{test}} \setminus \Ycal_{\mathrm{train}}|}{|\Ycal_{\mathrm{test}}|}\cdot 8c^2 \beta^{2\alpha}.
    \end{split}
\end{equation*}
Here, $2(\nu-1)$ becomes $2\nu-1$ as we also account for extra term from the decomposition in the first step of this proof. Finally, noting that $|\Ycal_{\mathrm{test}} \setminus \Ycal_{\mathrm{train}}|= |\Ycal_{\mathrm{test}}| - |\Ycal_{\mathrm{train}}|$ for $\Ycal_{\mathrm{train}} \subseteq \Ycal_{\mathrm{test}}$ completes our proof.

\end{proof}
\section{Proof of Corollary \ref{cor:uniform}}\label{appdx:proof_corr}

\begin{proof}
    \noindent
Let $y := (y_1, \ldots, y_d) \in \Ycal_{\mathrm{test}}$. By construction of the test grid, each coordinate can be written as $y_j = \frac{r_j}{N_2}$ for some $r_j \in \{0,1,\ldots,N_2 - 1\}$. Since $N_2 = m N_1$, we can further express
\[
y_j = \frac{r_j}{m} \cdot \frac{1}{N_1}.
\]
Now define a point $y' := (y_1', \ldots, y_d') \in \Ycal_{\mathrm{train}}$ by rounding each $r_j/m$ to its nearest integer:
\[
y_j' := \frac{[r_j/m]}{N_1},
\]
where $[r_j/m]$ denotes the closest integer to $r_j/m$ in $\{0, 1, \ldots, N_1 - 1\}$, with ties rounded down. Note that although $r_j/m$ can exceed $N_1 - 1$, we always map it to $N_1 - 1$, as $1 = N_1/N_1$ is not in the training grid.

By this construction, $y' \in \Ycal_{\mathrm{train}}$ is the nearest neighbor of $y$, and the Euclidean distance satisfies
\[
\beta := |y - y'|_2 \leq \frac{1}{N_1} \sqrt{ \sum_{j=1}^d \left| \frac{r_j}{m} - \left[ \frac{r_j}{m} \right] \right|^2 } \leq \frac{\sqrt{d}}{N_1}.
\]

\medskip 
Next, to bound $\nu$, consider how many test points $y \in \Ycal_{\mathrm{test}}$ can be mapped to a fixed training point $y' \in \Ycal_{\mathrm{train}}$. Fix a coordinate $j \in \{1, \ldots, d\}$, and suppose $y_j' = \frac{t_j}{N_1}$, where $t_j \in \{0, \ldots, N_1 - 1\} \setminus \{N_1-1\}$. Then $y_j = \frac{r_j}{m} \cdot \frac{1}{N_1}$ is mapped to $y_j'$ if and only if
\[
\frac{r_j}{m} \in \left(t_j - \frac{1}{2},\, t_j + \frac{1}{2} \right] \quad \Longleftrightarrow \quad r_j \in \left( m t_j - \frac{m}{2},\, m t_j + \frac{m}{2} \right].
\]
Since $r_j \in \{0, \ldots, N_2 - 1\}$, this interval contains at most $m$ integer values. However, in the edge case when $t_j = N_1 - 1$, the interval becomes
\[
r_j \in \left( m(N_1 - 2) + \frac{m}{2},\, m N_1 - 1 \right],
\]
which contains at most $1.5m$ integer values.

Therefore, for each coordinate, there are at most $1.5m$ admissible values of $r_j$, so the total number of test points assigned to a single training point is bounded by
\[
\nu \leq (1.5m)^d.
\]

Finally, noting 
\[\frac{|\Ycal_{\mathrm{train}}|}{|\Ycal_{\mathrm{test}}|} = \frac{1}{m^d}\]
and 
\[2\nu-1 = 2 (1.5m)^d -1 \leq 2^{d+1}\]
completes our proof.
\end{proof}

\section{Refined Bound for Uniform Grid}\label{appdx:refined}

Recall that in the proof of Corollary~\ref{thm:ub}, each test point $ y \in \Ycal_{\mathrm{test}} \setminus \Ycal_{\mathrm{train}} $ is assigned to its nearest neighbor in $ \Ycal_{\mathrm{train}} $. In a uniform grid, this nearest-neighbor mapping can lead to edge cases near the boundary (e.g., near $ y_j = 1 $), where a disproportionate number of test points may be mapped to a single training point. To address this, one can modify the assignment rule to use a slightly adjusted rounding scheme rather than strict nearest-neighbor mapping, ensuring a more balanced reuse across the grid.

To that end, $y := (y_1, \ldots, y_d) \in \Ycal_{\mathrm{test}}$ such that $y_j = \frac{r_j}{mN_1}$ for some $r_j \in \{0,1,\ldots,N_2 - 1\}$. 
Now define a point $y' := (y_1', \ldots, y_d') \in \Ycal_{\mathrm{train}}$ by mapping each $y_j$ to
\[
y_j' := \frac{\floor{r_j/m}}{N_1}.
\]
Instead of mapping $r_j/m$ to the closest integer, we map it to its floor value, which is close but not the nearest neighbor. Clearly,
\[
|y - y'|_2 \leq \frac{1}{N_1} \sqrt{ \sum_{j=1}^d \left| \frac{r_j}{m} - \left\lfloor \frac{r_j}{m} \right\rfloor \right|^2 } \leq \frac{\sqrt{d}}{N_1}.
\]

\medskip 
Next, we bound the number of test points $y \in \Ycal_{\mathrm{test}}$ can be mapped to a fixed training point $y' \in \Ycal_{\mathrm{train}}$. Fix a coordinate $j \in \{1, \ldots, d\}$, and suppose $y_j' = \frac{t_j}{N_1}$, where $t_j \in \{0, \ldots, N_1 - 1\}$. Then $y_j = \frac{r_j}{m} \cdot \frac{1}{N_1}$ is mapped to $y_j'$ if and only if
\[
\frac{r_j}{m} \in \left[t_j, t_j+1\right) \quad \Longleftrightarrow \quad r_j \in \left[ m t_j, mt_j +m \right).
\]
Thus, $r$ can take  $\leq m$ values.

Therefore, for each coordinate, there are at most $m$ admissible values of $r_j$, so the total number of test points assigned to a single training point is bounded by
\[
m^d.
\]
Redoing the proof of Theorem \ref{thm:ub} where each $y \in  \Ycal_{\mathrm{test}}\setminus \Ycal_{\mathrm{train}}$ is mapped to $y^{\prime} \in \Ycal_{\mathrm{train}}$ yields the claimed bound of 
\[2\mathcal{E}_\mu(\widehat{\operatorname{F}}_n, \operatorname{G}, \Ycal_{\mathrm{train}}) + \left( 1-\frac{1}{m^d}\right) \cdot 8c^2\cdot  \left(\frac{\sqrt{d}}{N_1} \right)^{2\alpha}.\]

\section{Bound on $\nu_{k,\ell}$ for Non-uniform Grids}
\label{appdx:nonuniform}

In this section, we derive a bound on the load-balancing factor $\nu_{k,\ell}$ under general geometric conditions. Let $\Ycal_k \subseteq \Ycal_\ell \subseteq \Ycal \subseteq \mathbb{R}^d$ be two grids. Recall that
\[
\nu_{k,\ell}
:=
\max_{z \in \Ycal_k}
\#\{y \in \Ycal_\ell : \mathrm{nn}(y)=z\},
\]
where $\mathrm{nn}(y)$ denotes the nearest neighbor of $y$ in $\Ycal_k$.

\begin{lemma}
Let $\Ycal_k \subseteq \Ycal_\ell \subseteq \Ycal$ be finite sets. Let $h_k$ denote the fill distance of $\Ycal_k$ and $q_\ell$ the separation distance of $\Ycal_\ell$. Then
\[
\nu_{k,\ell}
\le
\left(1+\frac{h_k}{q_\ell}\right)^d.
\]
\end{lemma}

\begin{proof}
Fix $z \in \Ycal_k$ and define
\[
A_z := \{y \in \Ycal_\ell : \mathrm{nn}(y)=z\}.
\]
We will bound $|A_z|$ uniformly in $z$. If $y \in A_z$, then by definition of nearest neighbor,
\[
|y - z|_2 = \min_{z' \in \Ycal_k} |y - z'|_2 \le h_k.
\]
Hence,
\[
A_z \subseteq B(z,h_k).
\]

By definition of the separation distance $q_\ell$, any two distinct points $y,y' \in \Ycal_\ell$ satisfy
\[
|y - y'|_2 \ge 2q_\ell.
\]
Therefore, the open balls $\{B^\circ(y,q_\ell)\}_{y \in A_z}$ are pairwise disjoint.

Let $y \in A_z$ and $x \in B^\circ(y,q_\ell)$. Then
\[
|x - z|_2 \le |x - y|_2 + |y - z|_2 < q_\ell + h_k,
\]
so
\[
B^\circ(y,q_\ell) \subseteq B(z,h_k + q_\ell)
\quad \text{for all } y \in A_z.
\]
Since this is true for all $y$ and $z$, we have
\[
\bigcup_{y \in A_z} B^\circ(y,q_\ell) \subseteq B(z,h_k + q_\ell).
\]
Since the sets $\{B^\circ(y,q_\ell)\}_{y \in A_z}$ are pairwise disjoint,
\[
\operatorname{vol}\!\left(\bigcup_{y \in A_z} B^\circ(y,q_\ell)\right)
=
\sum_{y \in A_z} \operatorname{vol}(B^\circ(y,q_\ell))
=
|A_z| \cdot \operatorname{vol}(B^\circ(0,q_\ell)),
\]
where the second equality holds due to translation invariance of Lebesgue measure.
Since the union is contained in $B(z,h_k + q_\ell)$, we obtain
\[
|A_z| \cdot \operatorname{vol}(B^\circ(0,q_\ell))
\le
\operatorname{vol}(B(z,h_k + q_\ell)).
\]
Using translation invariance of Lebesgue measure and the fact that open and closed balls have the same volume,
\[
|A_z| \cdot \operatorname{vol}(B(0,q_\ell))
\le
\operatorname{vol}(B(0,h_k + q_\ell)).
\]
Finally, using the scaling of volume in $\mathbb{R}^d$, we conclude
\[
|A_z|
\le
\left(\frac{h_k + q_\ell}{q_\ell}\right)^d
=
\left(1+\frac{h_k}{q_\ell}\right)^d.
\]

Taking the maximum over $z \in \Ycal_k$ yields the desired bound on $\nu_{k,\ell}$.
\end{proof}

\section{Proof of Proposition \ref{prop:green}}\label{appdx:proof_green}

Let
\[
u(y) := \int_{\Omega} g(y, x)\, f(x)\, dx.
\]
Then for any $ y_1, y_2 \in \Omega $, we have
\begin{equation*}
\begin{split}
|u(y_1) - u(y_2)|
&= \left| \int_{\Omega} \big( g(y_1, x) - g(y_2, x) \big)\, f(x)\, dx \right| \\
&\leq \int_{\Omega} |g(y_1, x) - g(y_2, x)|\, |f(x)|\, dx \\
&\leq \left( \int_{\Omega} |g(y_1, x) - g(y_2, x)|^2 dx \right)^{1/2} \cdot \|f\|_{L^2} \\
&\leq c\, |y_1 - y_2|^{\alpha} \cdot \sqrt{\operatorname{vol}(\Omega)} \cdot \|f\|_{L^2}.
\end{split}
\end{equation*}

Here, $ c $ is a uniform bound on the Hölder coefficient of $ y \mapsto g(y,x) $, valid for almost every $ x \in \Omega $. Since $ \Omega $ is bounded, we conclude that $ u \in C^{0,\alpha}(\Omega) $, completing the proof.

\section{Proof of Proposition \ref{prop:holder-NO}}\label{appdx:proof_prop_NO}
\begin{proof}
Let 
\[w(y):= \sigma \Big(\int_{\mathcal{X}} k(y, x)\, v(x) \,dx + b(y) \Big). \]
Then, using the fact that $\sigma(\cdot)$ is $1$-Lipschitz, we have
\begin{equation*}
    \begin{split}
        |w(y_1) -w(y_2)| &= \left|\sigma \Big(\int_{\mathcal{X}} k(y_1, x)\, v(x) \,dx + b(y_1) \Big) - \sigma \Big(\int_{\mathcal{X}} k(y_2, x)\, v(x) \,dx + b(y_2) \Big) \right|\\
        &\leq \left|\int_{\mathcal{X}} k(y_1, x)\, v(x) \,dx + b(y_1)  - \int_{\mathcal{X}} k(y_2, x)\, v(x) \,dx - b(y_2)  \right|\\
        &= \left|\int_{\mathcal{X}} (k(y_1, x)-k(y_2, x))\, v(x) \,dx + b(y_1)  - b(y_2)  \right|\\
        &\leq \left|\int_{\mathcal{X}} (k(y_1, x)-k(y_2, x))\, v(x) \,dx \right| + \left|b(y_1)  - b(y_2)  \right|
    \end{split}
\end{equation*}
Note that $\left|b(y_1)  - b(y_2)  \right| \leq [b] \, |y_1-y_2|^{\alpha}$. Similarly, 
\begin{equation*}
    \begin{split}
   &\left|\int_{\mathcal{X}} (k(y_1, x)-k(y_2, x))\, v(x) \,dx \right|\\
   &= \sqrt{\sup_{x} |k(y_1, x)-k(y_2, x)|^2 \,\,  \text{vol}(\Xcal)}\, \, \, \norm{v}_{L^2}\\
   &= [k]\, |y_1-y_2|^{\alpha} \, \sqrt{\text{vol}(\Xcal)} \,\,\norm{v}_{L^2},
    \end{split}
\end{equation*}
where $[k]$ is supremum over all Holder coefficient $[k(\cdot, x)]$.  Recall that $[k]<\infty$ by assumption. Thus, combining everything, we have
\[   |w(y_1) -w(y_2)| \leq \left( [k]\,  \sqrt{\text{vol}(\Xcal)}\, \norm{v}_{L^2} + [b]\right) |y_1-y_2|^{\alpha}. \]
This completes our proof.
\end{proof}

\section{Proof of Theorem \ref{thm:ub-discreteinp}}
Note that
\begin{equation*}
    \begin{split}
           \Ecal_{\mu}(\widehat{\operatorname{F}}_n, \operatorname{G}, \Xcal_{\mathrm{test}}, \Ycal_{\mathrm{test}} )= \expect_{v \sim \mu} \left[ \frac{1}{|\Ycal_{\mathrm{test}}|} \sum_{y \in \Ycal_{\mathrm{test}}} \left( \widehat{\operatorname{F}}_n(v_{|\Xcal_{\mathrm{test}}})(y) - \operatorname{G}(v)(y) \right)^2 \right].
    \end{split}
\end{equation*}
We can write 
\begin{equation*}
    \begin{split}
       &|\widehat{\operatorname{F}}_n(v_{|\Xcal_{\mathrm{test}}})(y) - \operatorname{G}(v)(y)| \\
       &= |\widehat{\operatorname{F}}_n(v_{|\Xcal_{\mathrm{test}}})(y) - \widehat{\operatorname{F}}_n(v_{|\Xcal_{\mathrm{train}}})(y) +\widehat{\operatorname{F}}_n(v_{|\Xcal_{\mathrm{train}}})(y) -\operatorname{G}(v)(y)|\\
       &\leq |\widehat{\operatorname{F}}_n(v_{|\Xcal_{\mathrm{test}}})(y) - \widehat{\operatorname{F}}_n(v_{|\Xcal_{\mathrm{train}}})(y)| + |\widehat{\operatorname{F}}_n(v_{|\Xcal_{\mathrm{train}}})(y) -\operatorname{G}(v)(y)|
    \end{split}
\end{equation*}
Thus, using the inequality $(a+b)^2 \leq 2(a^2+b^2)$, we have
\begin{equation*}
    \begin{split}
         &\Ecal_{\mu}(\widehat{\operatorname{F}}_n, \operatorname{G}, \Xcal_{\mathrm{test}}, \Ycal_{\mathrm{test}} ) \\
         &= 2\expect_{v \sim \mu} \left[ \frac{1}{|\Ycal_{\mathrm{test}}|} \sum_{y \in \Ycal_{\mathrm{test}}} |\widehat{\operatorname{F}}_n(v_{|\Xcal_{\mathrm{test}}})(y) - \widehat{\operatorname{F}}_n(v_{|\Xcal_{\mathrm{train}}})(y)| ^2 \right] \\
         &+ 2\expect_{v \sim \mu} \left[ \frac{1}{|\Ycal_{\mathrm{test}}|} \sum_{y \in \Ycal_{\mathrm{test}}} |\widehat{\operatorname{F}}_n(v_{|\Xcal_{\mathrm{train}}})(y) -\operatorname{G}(v)(y)|^2 \right]
    \end{split}
\end{equation*}

For the second term, using the same arguments as in the proof of Theorem~\ref{thm:ub}, we obtain
\begin{equation*}
\begin{split}
\expect_{v \sim \mu} \left[ \frac{1}{|\Ycal_{\mathrm{test}}|} \sum_{y \in \Ycal_{\mathrm{test}}} \left|\widehat{\operatorname{F}}_n(v_{|\Xcal_{\mathrm{train}}})(y) - \operatorname{G}(v)(y)\right|^2 \right]
&\leq \frac{|\Ycal_{\mathrm{train}}|}{|\Ycal_{\mathrm{test}}|} \cdot (2\nu - 1) \cdot \Ecal_{\mu}(\widehat{\operatorname{F}}_n, \operatorname{G}, \Xcal_{\mathrm{train}}, \Ycal_{\mathrm{train}}) \\
&\quad + \frac{|\Ycal_{\mathrm{test}} \setminus \Ycal_{\mathrm{train}}|}{|\Ycal_{\mathrm{test}}|} \cdot 8c^2 \beta^{2\alpha}.
\end{split}
\end{equation*}
This completes the contribution of the second term in the bound stated in Theorem~\ref{thm:ub-discreteinp}.

So, it remains to bound
\[2\expect_{v \sim \mu} \left[ \frac{1}{|\Ycal_{\mathrm{test}}|} \sum_{y \in \Ycal_{\mathrm{test}}} |\widehat{\operatorname{F}}_n(v_{|\Xcal_{\mathrm{test}}})(y) - \widehat{\operatorname{F}}_n(v_{|\Xcal_{\mathrm{train}}})(y)| ^2 \right].\]
Again using $ (a+b)^2 \leq 2a^2 + 2b^2 $, we can further write
\begin{equation*}
\begin{split}
&2\expect_{v \sim \mu} \left[ \frac{1}{|\Ycal_{\mathrm{test}}|} \sum_{y \in \Ycal_{\mathrm{test}}} \left( \widehat{\operatorname{F}}_n(v_{|\Xcal_{\mathrm{test}}})(y) - \widehat{\operatorname{F}}_n(v_{|\Xcal_{\mathrm{train}}})(y) \right)^2 \right] \\
&\leq 4 \expect_{v \sim \mu} \left[ \frac{1}{|\Ycal_{\mathrm{test}}|} \sum_{y \in \Ycal_{\mathrm{test}}} \left( \widehat{\operatorname{F}}_n(v)(y) - \widehat{\operatorname{F}}_n(v_{|\Xcal_{\mathrm{test}}})(y) \right)^2 \right] \\
&\quad + 4 \expect_{v \sim \mu} \left[ \frac{1}{|\Ycal_{\mathrm{test}}|} \sum_{y \in \Ycal_{\mathrm{test}}} \left( \widehat{\operatorname{F}}_n(v)(y) - \widehat{\operatorname{F}}_n(v_{|\Xcal_{\mathrm{train}}})(y) \right)^2 \right]\\
&\leq 4(\varepsilon_{k_1} +\varepsilon_{k_2}).
\end{split}
\end{equation*}
The final step uses the fact that $\Xcal_{\mathrm{train}}=\Xcal_{k_1}$ and $\Xcal_{\mathrm{test}}=\Xcal_{k_2}$.

\end{document}